\documentclass{article} 
\usepackage{iclr2025_conference,times}

\usepackage{amsmath,amsfonts,bm}

\def\eqref#1{equation~\ref{#1}}

\def\1{\bm{1}}

\DeclareMathAlphabet{\mathsfit}{\encodingdefault}{\sfdefault}{m}{sl}
\SetMathAlphabet{\mathsfit}{bold}{\encodingdefault}{\sfdefault}{bx}{n}

\newcommand{\lr}{\alpha}

\DeclareMathOperator*{\argmin}{arg\,min}

\newcommand{\chapternote}[1]{{%
  \let\thempfn\relax
  \footnotetext[0]{\emph{#1}}
}}

\newif{\ifanswer}

\usepackage{textcomp}
\usepackage{xcolor}
\usepackage{xspace}
\usepackage{bbm}
\usepackage{makecell}

\definecolor{orange(sae/ece)}{rgb}{1.0, 0.49, 0.0}
\definecolor{teal(sae/ece)}{rgb}{0, 0.47, 0.52}
\definecolor{purple}{rgb}{0.74, 0.65, 1.0}
\definecolor{dark_purple}{rgb}{0.72, 0.33, 0.82}
\definecolor{light_gray}{rgb}{0.9, 0.9, 0.9}
\definecolor{medium_gray}{rgb}{0.6, 0.6, 0.6} 
\definecolor{dark_gray}{rgb}{0.2, 0.2, 0.2} 
\definecolor{dark_blue}{rgb}{0.098, 0.239, 0.52}
\definecolor{dark_brown}{rgb}{0.3255, 0.004, 0.001}
\definecolor{r3mcolor}{rgb}{0.478, 0.1569, 0.4863}
\definecolor{light_blue}{rgb}{0.33, 0.80, 1}
\definecolor{mauve}{rgb}{0.651, 0.369, 0.616}
\definecolor{light_green}{rgb}{0.278, 0.659, 0.075}

\newcommand{\iqt}{\textcolor{black}{IQT}\xspace}
\newcommand{\iqtpi}{\textcolor{mauve}{\textbf{IQT-pi}}\xspace}
\newcommand{\iqtpd}{\textcolor{orange}{\textbf{IQT-pd}}\xspace}

\newcommand{\ensemble}{\textcolor{gray}{\textbf{EnsembleDAgger}}\xspace}
\newcommand{\cdagger}{\textcolor{orange}{\textbf{ConformalDAgger}}\xspace}
\newcommand{\lazy}{\textcolor{light_green}{\textbf{LazyDAgger}}\xspace}
\newcommand{\safe}{\textcolor{red}{\textbf{SafeDAgger}}\xspace}

\newcommand{\edit}[1]{{\color{black}{#1}}}

\newcommand{\para}[1]{\smallskip \noindent \textbf{{#1}.}}

\newcommand{\dcalib}{\DD_\mathrm{calib}}
\newcommand{\estqlo}{\hat{q}_{\alpha_{lo}}}
\newcommand{\estqhi}{\hat{q}_{\alpha_{hi}}}

\newcommand{\errt}{\text{err}_t}

\newcommand{\yhat}{\hat{y}_t}
\newcommand{\lrt}{\gamma_t}
\newcommand{\dlnorm}{||\Delta_{1:T}||_1}
\newcommand{\dl}{\Delta}

\newcommand{\N}{\mathcal{N}}
\newcommand{\expert}{\pi^{h}}
\newcommand{\nov}{\pi^{r}}

\newcommand{\actexp}{a^{h}}
\newcommand{\actnov}{a^{r}}

\newcommand{\gmtp}{\frac{\gamma}{M_{T+1}}}

\def\PP{\mathcal{P}}
\def\RR{\mathbb{R}}

\def\AA{\mathcal{A}}
\def\PP{\mathbb{P}}
\def\EE{\mathbb{E}}

\def\XX{\mathcal{X}}
\def\YY{\mathcal{Y}}
\def\NN{\mathbb{N}}
\def\TT{\mathcal{T}}

\def\DD{\mathcal{D}}

\def\err{\text{err}}
\def\ind{\mathbbm{1}}

\def\err{\text{err}}
\def\inf{\text{inf}}
\def\sup{\text{sup}}

\def\obs{\text{obs}}
\def\obst{\text{obs}_t}
\def\lr{\gamma}
\def\constlr{\text{lr}}

\usepackage{amsthm}

\newtheorem{proposition}{Proposition}
\newtheorem{lemma}{Lemma}

\usepackage[noend]{algpseudocode}
\usepackage{hyperref}
\usepackage{url}

\usepackage[utf8]{inputenc} 
\usepackage[T1]{fontenc}    
\usepackage{hyperref}       
\usepackage{url}            
\usepackage{booktabs}       
\usepackage{amsfonts}       
\usepackage{nicefrac}       
\usepackage{microtype}      
\usepackage{xcolor}         

\usepackage{times}
\usepackage{amssymb}
\usepackage{mathtools}
\usepackage{algorithm}
\usepackage{algpseudocode}
\usepackage{xcolor}
\usepackage{booktabs}
\usepackage{enumitem}
\usepackage{bbm}
\usepackage{subfig}
\usepackage{wrapfig}

\title{Conformalized Interactive Imitation Learning: \smash{Handling Expert Shift \& Intermittent Feedback}}

\author{
Michelle Zhao\textsuperscript{a}, Reid Simmons\textsuperscript{a}, Henny Admoni\textsuperscript{a}, Aaditya Ramdas\textsuperscript{\textdagger b}, Andrea Bajcsy\textsuperscript{\textdagger a} \\
\textsuperscript{a} Robotics Institute, School of Computer Science, Carnegie Mellon University\\
\textsuperscript{b} Departments of Statistics and Machine Learning, Carnegie Mellon University\\
\texttt{\{mzhao2, hadmoni, rsimmons, aramdas, abajcsy\}@andrew.cmu.edu} \\
}

\iclrfinalcopy 
\begin{document}

\maketitle

\begin{abstract}


In interactive imitation learning (IL), uncertainty quantification offers a way for the learner (i.e. robot) to contend with distribution shifts encountered during deployment by actively seeking additional feedback from an expert (i.e. human) online.
Prior works use mechanisms like ensemble disagreement or Monte Carlo dropout to quantify when black-box IL policies are uncertain; however, these approaches can lead to overconfident estimates when faced with deployment-time distribution shifts. 
Instead, we contend that we need uncertainty quantification algorithms that can leverage the expert human feedback received \textit{during deployment time} to adapt the robot's uncertainty \textit{online}. 
To tackle this, we draw upon online conformal prediction, a distribution-free method for constructing prediction intervals online given a stream of ground-truth labels. 
Human labels, however, are intermittent in the interactive IL setting. 
Thus, from the conformal prediction side, we introduce a novel uncertainty quantification algorithm called intermittent quantile tracking (IQT) that leverages a probabilistic model of intermittent labels, maintains asymptotic coverage guarantees, and empirically achieves desired coverage levels. 
From the interactive IL side, we develop ConformalDAgger, a new approach wherein the robot uses prediction intervals calibrated by IQT as a reliable measure of deployment-time uncertainty to actively query for more expert feedback. 
We compare ConformalDAgger to prior uncertainty-aware DAgger methods in scenarios where the distribution shift is (and isn't) present because of changes in the expert's policy. 
We find that in simulated and hardware deployments on a 7DOF robotic manipulator, ConformalDAgger detects high uncertainty when the expert shifts and increases the number of interventions compared to baselines, allowing the robot to more quickly learn the new behavior. 
Project page at \url{cmu-intentlab.github.io/conformalized-interactive-il/}.

\end{abstract}

\vspace{-1.3em}
\section{Introduction}
\vspace{-0.9em}
\label{sec:introduction}

End-to-end robot policies trained via imitation learning (IL) have proven to be an extremely powerful way to learn complex robot behaviors from expert human demonstrations \citep{schaal1996learning, price2003accelerating, argall2009survey, levine2016end, jang2022bc, chi2023diffusion, kim2024openvla}.
At the same time, distribution shift is a core challenge in this domain, hampering the reliability of deploying such robot policies in the real world \citep{chang2021mitigating}. 

One way to combat this is via uncertainty quantification. By training an ensemble of policies \citep{menda2019ensembledagger} or via monte-carlo dropout during training \citep{cui2019uncertainty}, the robot learner can detect uncertain states and \textit{actively elicit} additional action labels from the human expert online via an interactive IL framework (such DAgger \citep{ross2011reduction}). 
At their core, these prior uncertainty-aware IL methods look to the training demonstration data as a proxy for deployment-time uncertainty. 
Any human expert labels requested at deployment time using this uncertainty estimate are simply stored for later re-training; the uncertainty estimate itself is not adapted online to the expert data nor does it inform any subsequent queries during the deployment episode.

Instead, we contend that the human feedback requested and received \textit{during} deployment time is a valuable uncertainty quantification signal that should be leveraged to \textit{update} the robot's uncertainty estimate \textit{online}. If properly accounted for, the updated uncertainty estimate will influence when the robot asks for more help, enabling it to targetedly probe the human expert to improve policy performance. The challenge is how to do this online uncertainty update in the presence of the end-to-end ``black box'' policies underlying modern imitation learning. 

To tackle this, we take inspiration from online conformal prediction \citep{gibbs2021adaptive} which is a distribution-free way to represent uncertainty via prediction sets constructed on the output of a black-box model. 
Our first contribution is extending online conformal prediction to the case where labels are observed intermittently, as is the case in interactive IL with an expert. 
Specifically, we instantiate Intermittent Quantile Tracking (IQT), an algorithm which adjusts prediction intervals online to ensure that the true label lies within the predicted interval with high-probability despite probabilistic access to labels. 
On standard conformal time series datasets, we empirically find that IQT achieves empirical coverage close to the desired level by boosting the size of the calibrated intervals based on the likelihood of observing feedback. 

With our intermittent conformal algorithm in hand, we develop ConformalDAgger, a new interactive IL approach wherein the robot uses prediction intervals calibrated by IQT as a reliable measure of deployment-time uncertainty to actively query for more expert feedback. 
We instantiate ConformalDAgger in a simulated 4D robot goal-reaching task and in hardware on a 7 degree-of-freedom robotic manipulator that uses a state-of-the-art Diffusion Policy \citep{chi2023diffusion} learned via IL to perform a sponging task (Figure~\ref{fig:real_robot_intuition}). 
We study how ConformalDAgger compares to prior uncertainty-aware DAgger methods when the deployment-time and training-time datasets are from the same distribution (left, Figure~\ref{fig:real_robot_intuition}) as well as a shifted one (right, Figure~\ref{fig:real_robot_intuition}). 
Specifically, we instantiate a potential source of distribution shift as \textit{expert policy shift}: the training-time expert wipes a line-drawing with a straight-line path while the deployment-time expert refines their strategy to be a zigzag path (Figure~\ref{fig:real_robot_intuition}).
We find that ConformalDAgger automatically increases uncertainty online when the expert shifts, resulting in more expert labels queries compared to EnsembleDAgger and allowing our approach to rapidly learn a policy aligned with the expert's intentions.

\begin{figure}[t]
    \centering
    \includegraphics[width=0.99\linewidth]{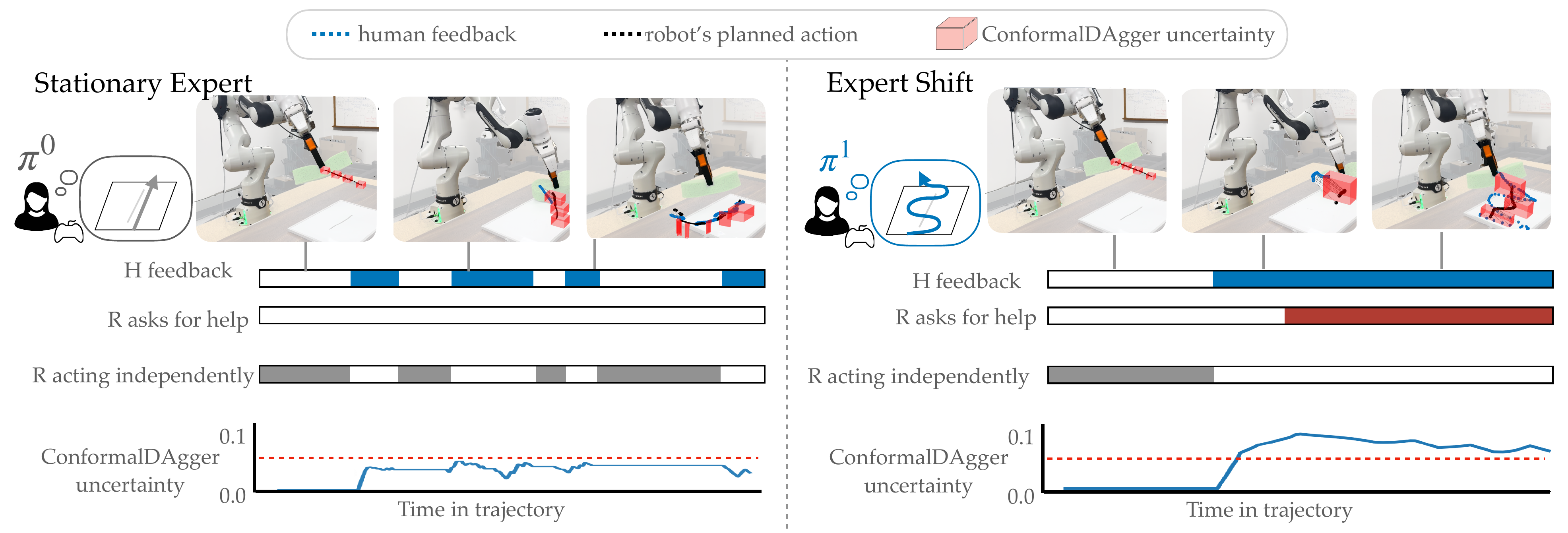}
    \caption{\textbf{Conformalized Interactive Imitation Learning.} The robot learns an initial policy via imitation learning to wipe a line drawing by following a straight path. When the initial policy is deployed, ConformalDAgger calibrates the robot's uncertainty (represented by red prediction interval boxes) based on feedback received from the human during the interactive IL loop. (left) The robot refrains from asking questions when its uncertainty is low because the expert policy at deployment time is aligned to the demonstrations on which the learner was trained. (right) When the human shifts in their task strategy, uncertainty increases and the robot starts to query for more expert labels.}  
    \label{fig:real_robot_intuition}
    \vspace{-1.5em}
\end{figure}

\vspace{-0.7em}
\section{Related Work}
\vspace{-0.7em}
\label{sec:related_work}

\noindent \textbf{Interactive Imitation Learning (IL) with Online Experts.}

Interactive IL is a branch of imitation learning wherein a robot learner can query a (human) expert to receive additional labels either during or after task execution \citep{celemin2022interactive}. 
\edit{The utility of in-the-loop expert interventions is mitigating distribution shift between training and deployment of the policy. Distribution shift can stem from a multitude of reasons, including covariate shift \citep{spencer2021feedback} (e.g., encountering new observations of the environment during policy execution), or expert shift: wherein the expert’s latent strategy changes over time \citep{hong2024learning, sagheb2023towards, xie2021learning} A foundational approach for interactive IL is DAgger (Dataset Aggregation) \citep{ross2011reduction}, which iteratively augments the training dataset by aggregating data from the expert and learner policies and assuming the expert is stationary. }
In the online case, the robot learner can cede control to the expert at any time to get additional state-action data (also known as \textit{robot-gated} feedback) or the expert can actively intervene at any time (also known as \textit{human-gated} feedback) \citep{kelly2019hg}. 
From the learner's perspective, a key question is when to request feedback from the expert so that it minimizes expert effort but also minimizes negative events caused by an erroneous policy (e.g., running into a wall). 
On one hand, prior works focus on minimizing human effort by constraining robot requests via a limited human attention model \citep{hoque2023fleet} or a budget of human interventions \citep{hoque2021thriftydagger, hoque2021lazydagger}. 
On the other hand, prior works prioritize deployment-time safety by classifying safe versus unsafe states (SafeDAgger \citep{zhang2016query}, \edit{Replay Estimation \citep{swamy2022minimax}}), estimating uncertainty via ensemble disagreement (EnsembleDAgger \citep{menda2019ensembledagger}), 
or Monte Carlo dropout \citep{cui2019uncertainty}. 
We present ConformalDAgger, which uses our novel uncertainty quantification method grounded in conformal prediction to adaptively increase the learner requests for help under uncertainty and decrease the number of requests when confident.

\vspace{-0.2em}
\para{Online Conformal Prediction}
Conformal prediction is a distribution-free uncertainty quantification method for constructing prediction intervals for both classification and regression problems \citep{angelopoulos2021gentle, romano2019conformalized, romano2020classification, zaffran2022adaptive}, as well as for offline or online data. 
We focus on the \textit{online} setting (e.g., timeseries) where uncertainty quantification is performed on streaming pairs of input-label data that are not necessarily i.i.d. \citep{gibbs2021adaptive}. 
Broadly speaking, there are two predominant algorithms in this setting: adaptive conformal inference (ACI) \citep{gibbs2021adaptive, gibbs2024conformal, bhatnagar2023improved, zaffran2022adaptive} and quantile tracking  (QT) \citep{angelopoulos2024conformal}. 
Both are online gradient descent-based methods which guarantee asymptotic coverage in the online setting. 
Unlike ACI which is prone to infinitely sized intervals after a series of miscoverage events, QT directly estimates the value of the empirical quantile itself, ensuring coverage with finite intervals. In this work, we relax the assumption that labels must be observed at each time point in the streaming data and extend the online conformal paradigm to ensure coverage in the intermittent label regime.

\vspace{-0.7em}
\section{Online Conformal Prediction with Intermittent Labels}
\vspace{-0.7em}
\label{sec:intermittent_quantile}

From the uncertainty quantification side, our core technical contribution is extending online conformal prediction to settings where ground truth labels are intermittently observed. 
We present our algorithm in the context of quantile tracking for relevance to our interactive IL experiments. 
However, we also derive an extension of ACI \citep{gibbs2021adaptive} to intermittent labels in the Appendix Section \ref{sec:iaci_proofs}. 

\vspace{-0.2em}

\begin{wrapfigure}{r}{0.4\textwidth}
\vspace{-2em}
  \begin{center}
  \includegraphics[width=0.3\textwidth]{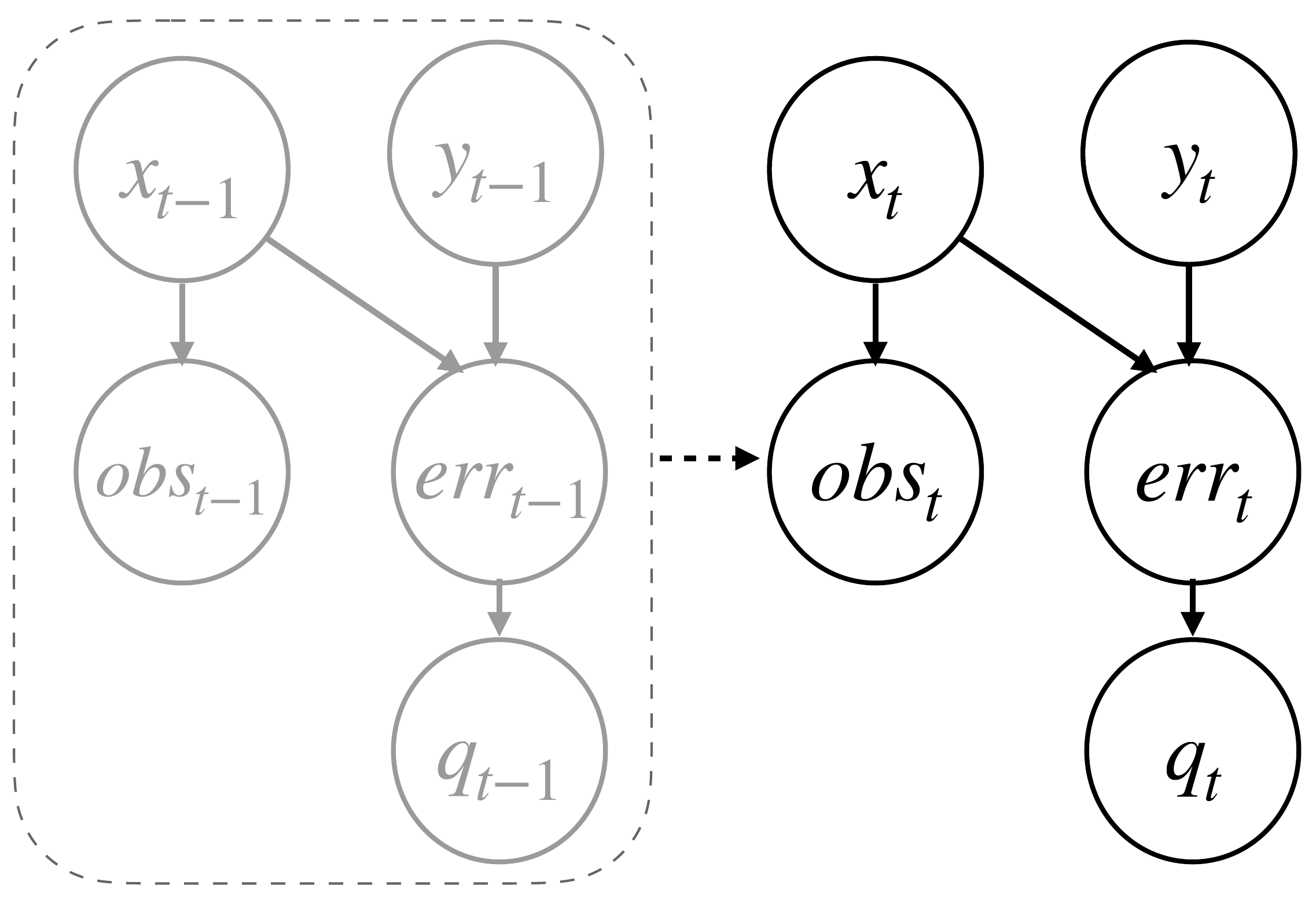}
  \end{center}
  \vspace{-1.4em}
  \caption{\edit{\textbf{Graphical Model of IQT.} IQT introduces random variable $\obst$, which represents the receiving of ground truth observations at $t$. $\obst$ is distributed according to $p_t$, which may depend on $x_t$ or history through $t-1$.}}
  \label{fig:scm_iqt}
  \vspace{-1.5em}
\end{wrapfigure}
\textbf{Setting.} We focus on online conformal prediction in the adversarial setting, such as time-series forecasting.
This considers an \textit{arbitrary} sequence of data points $(x_t, y_t) \in \XX \times \YY$, for $t = 1, 2,...$, that are not necessarily I.I.D. 
Our goal is to produce prediction sets on the output of any base prediction model such that the sets contain the true label with a specified coverage probability.  
Mathematically, at each time $t$, we observe $x_t$ and seek to cover the true label $y_t$ with a set $C_t(x_t)$, which depends on a base prediction model, $\hat{f}: \XX \rightarrow \YY$. 
The base model takes as input the current $x_t$ and  
outputs prediction $\yhat$; in the \textit{non}-intermittent case, we observe the ground-truth label $y_t$ after each prediction. 

%
%

\para{Background: Quantile Tracking} 
The quantile tracking (QT) algorithm from \cite{angelopoulos2024online} implicitly 
seeks to track the value of the  $1-\alpha$ quantile via online gradient descent on the quantile loss \citep{koenker1978regression}.
As in \cite{angelopoulos2024online}, we leverage a bounded nonconformity score function: $s: \XX \times \YY \rightarrow [0,B]$ where $0 < B < \infty$, to quantify the error made by the initial prediction of the base model. 
We assume the nonconformity score function $s(x_t,y_t)$ is negatively oriented (lower values indicate less nonconformity or greater prediction accuracy). 
Let $q_t$ represents the estimated $1-\alpha$ quantile of the score sequence $s_t, t \in \NN$. 
Prediction intervals are constructed using the nonconformity score function:
\begin{equation}
    C_t(x_{t}) = \{y \in \YY: s(x_t, y_t) \leq q_t\}
\label{eq:interval_qt}
\end{equation}
To expand or contract the prediction intervals, the level $q_t$ is adjusted via the online update:
\begin{equation}
    q_{t+1} = q_t + \lrt (\errt - \alpha),
\label{eq:online_update_qt}
\end{equation}
where $\lrt > 0$ is a  time-varying step size and $\errt  = \ind_{y_t \notin C_t(x_t)}$ is the empirical miscoverage at $t$.

Intuitively, this update increases the quantile threshold when the model continuously miscovers and decreases it when the model coverage is performant. For arbitrary step size $\lrt$ with no assumptions on the sequence of data points $(x_1, y_1), (x_2,y_2), ...$ and an initial quantile threshold $q_1 \in [0,B]$, the quantile tracking update satisfies Equation \ref{eq:qt_coverage_guarantee} (Theorem 2 of \cite{angelopoulos2024online}):
\begin{equation}
    \left\vert \frac{1}{T} \sum_{t=1}^T \ind_{y_t \in C_t(x_t)} - (1-\alpha)   \right\vert \leq \frac{B + \max_{1\leq t \leq T} \lrt}{T} \cdot \dlnorm
\label{eq:qt_coverage_guarantee}
\end{equation}
where $\dl$ is defined $\dl_1 = \lr_1^{-1}$ and $\dl_t = \lr_t^{-1} - \lr_{t-1}^{-1}$ for $t\geq 2$, and $\dl_{1:T} = (\dl_1,\dots,\dl_T)$.
We can see that $\lim_{T\rightarrow \infty} \frac{1}{T} \sum_{t=1}^T 1 - \errt$ approaches $1 - \alpha$. This guarantees quantile tracking gives the $1-\alpha$ long-term empirical coverage frequency. 

\vspace{-0.2em}
\para{Intermittent Quantile Tracking (IQT)}
Our paradigm lifts the assumption that the ground truth label $y_t$ is observed constantly. Instead, it is observed with some probability at each timestep. 
Let the binary random variable $\obst \in \{0,1\}$ represent whether the robot observes label $y_t$ at timestep $t$:
\begin{equation}
  \obst \coloneqq \begin{cases}
    1, & \text{ if } y_t \text{ observed } \\
    0, & \text{otherwise }.
  \end{cases}
\label{eq:obst_def}
\end{equation}
We introduce a probabilistic observation model, $p_t \coloneqq \PP(\obst=1 \mid x_t)$, where $p_t$ represents the probability of receiving feedback, which can be dependent on input $x_t$ (and past history, this possibility is demonstrated by the dotted line in Figure \ref{fig:scm_iqt}). A key observation under the paradigm of intermittent labels is that $\errt$ may not at every timestep be accessible to the algorithm if $y_t$ is not provided, but the value of $\errt$ exists at every timestep even if unobservable.

Our quantile tracking update under probabilistic observations, which we call Intermittent Quantile Tracking (IQT), is defined in Equation \ref{eq:q_update_iqt} as:
\begin{equation}
    q_{t+1} =  q_t +  \frac{\lrt}{p_t}(\errt - \alpha) \obst
\label{eq:q_update_iqt}
\end{equation}

\begin{proposition}
\label{prop:iqt_coverage}
Let $(x_1, y_1), (x_2, y_2), ...$ be an arbitrary sequence of data points, and let $s: \XX \times \YY \rightarrow [0,B]$. Let $\gamma_t$ be an arbitrary positive sequence, and fix an initial threshold $q_1 \in [0,B]$. Then Intermittent Quantile Tracking (IQT) satisfies for all $T\geq 1$:
\begin{equation}
    \left\vert\frac{1}{T} \sum_{t=1}^T  \errt - \alpha \right\vert \leq \frac{B + \max_{1\leq t \leq T}\frac{\lrt}{p_t}}{T} \cdot \dlnorm
\end{equation}
\end{proposition}
\edit{This guarantees IQT gives the $1-\alpha$ long-term empirical coverage frequency, even with intermittent ground truth access, regardless of the underlying data generation process. While these asymptotic coverage guarantees are sound in theory, we believe it is important to acknowledge that typical robot deployment conditions are finite-horizon. Nevertheless our empirical findings in Section~\ref{sec:experiments_dagger} indicate the practical utility of IQT's uncertainty estimates to obtain extra human feedback during deployment.} Refer to full proof in Appendix Section \ref{sec:iqt_proofs}. Note that when feedback is constant (i.e. $p_t = 1$), $\obst=1$ for all timesteps, and IQT reduces to online quantile tracking with arbitrary step sizes.




\textit{Practical Note: Choosing Gamma.} In practice, an important decision is the choice of $\lrt$. Prior works in quantile tracking with constant labels \citep{angelopoulos2024conformal} choose $\lrt = \constlr \hat{B}_t$, where $\hat{B}_t := \max_{t-k \leq r \leq t-1} s(x_r, y_r)$, $k$ represents a look-back window across the previous timesteps with observed ground truth labels, and $\constlr$ is a small constant. On our domains, we find that choosing $\lrt$ such that it contains $\hat{B}_t$ improves empirical coverage over the choice of a constant $\lrt = \gamma$.


\vspace{-0.7em}
\subsection{Experiments: Intermittent Quantile Tracking on Standard Timeseries Data}
\vspace{-0.7em}
\label{sec:experiments_validation}
\begin{figure}[t]
    \centering
    \includegraphics[width=0.99\linewidth]{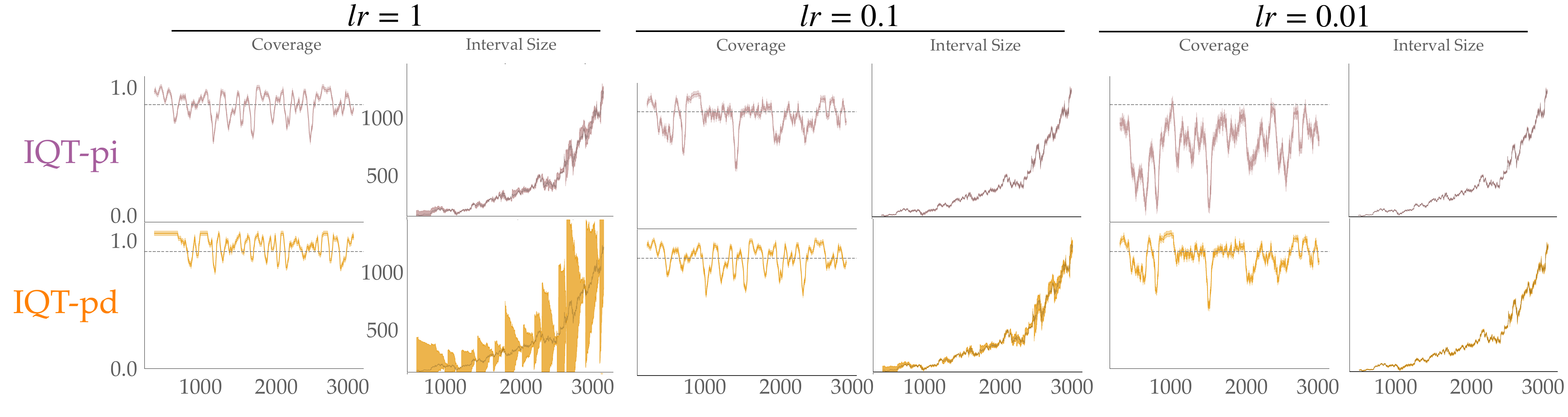}
    \caption{\textbf{IQT on Amazon Stocks with AR base model: Coverage \& Interval Visualization.} We set $p_t = 0.1, \forall t$ to see true price labels only 10\% of the time and show the prediction interval for 1 seed and coverage averaged over 5 seeds (shaded error on coverage plots show std deviation). \iqtpd is prone to large intervals when $\constlr$ is large, and boosts coverage over \iqtpi when $\constlr$ is small.}
    \label{fig:amzn_p01}
    \vspace{-1.5em}
\end{figure}

We first empirically evaluate our approach for intermittent conformal quantile tracking on time series benchmarks used in the online conformal prediction literature. The goal of testing on standard conformal datasets is to evaluate (1) how different choices of step size $\lrt$ affect the empirical coverage of \iqt, and (2) validate that under probabilistically intermittent observations, \iqt maintains coverage close to the desired level. We use the findings from these experiments to inform how we leverage intermittent quantile tracking in the interactive imitation learning domain.

\vspace{-0.2em}
\para{Setup} We test on three benchmark datasets from \citet{angelopoulos2024conformal}: (1) Amazon stock prices, (2) Google stock prices \citep{nguyen2018stock}, and the (3) Elec2 dataset \citep{harries1999splice}. We test four base prediction models, $\hat{f}$, all trained via darts \citep{herzen2022darts} to see if \iqt consistently maintains coverage close to the desired level. We present the Autoregressive (\textbf{AR}) model with 3 lags for brevity in the main text and defer the other model results to the Appendix. Our nonconformity score is the asymmetric (signed) residual score. We measure (1) marginal coverage over the time series, (2) longest miscoverage error sequence, and (3) mean prediction interval size.



\vspace{-0.2em}
\para{Observation Models} We test three different observation frequencies: (1) \textit{infrequent} $p_t=0.1, \forall t$, (2) \textit{partial} $p_t=0.5, \forall t$, and (3) \textit{frequent} $p_t=0.9, \forall t$. We focus in this section on the results for \textit{infrequent} observations, and discuss further results in the Appendix Section \ref{sec:extended_timeseries_iqt}.

\vspace{-0.2em}
\para{Methods}
We compare two variants of our \iqt algorithm by controlling if time-varying step-size $\lrt$ in Equation~\ref{eq:q_update_iqt} depends or does not depend on $p_t$.
For each variant, we test $\constlr \in [1,0.1,0.01]$.
\textbf{IQT with a $p_t$-independent update} (\iqtpi) uses $\lrt = \constlr \hat{B}_t p_t$. The quantile tracking update reduces to $q_{t+1} =  q_t +  \constlr \hat{B}_t(\errt - \alpha) \obst$. 
Since the $p_t$ term cancels out in the update, we refer to this model as the p-\textit{independent} update.
\textbf{IQT with a $p_t$-dependent update} (\iqtpd) uses $\lrt = \constlr \hat{B}_t$. 
Since the $p_t$ term remains in the update $q_{t+1} =  q_t +  \frac{\constlr \hat{B}_t}{p_t}(\errt - \alpha) \obst$, we refer to this model as the p-dependent update. For both variants, we set our desired coverage to $1-\alpha=0.9$ for all experiments. We set lookback window $k=100$ timesteps for the Amazon stock price data.



\para{Results}
We center our discussion on the Amazon stock price results for \iqtpd and \iqtpi with AR base model under ground truth labels observed with $p_t=0.1$ frequency, as the intermittent setting is our focus. See Appendix Section \ref{sec:extended_timeseries_iqt} for further results on other base models and datasets. Inversely scaling the quantile tracking step size by $p_t$ causes \iqtpd to construct larger intervals than \iqtpi. Figure \ref{fig:amzn_p01} shows the prediction coverage as a moving average (window=50) for 5 random seeds and interval sizes for one seed. When $\constlr$ is high (1.0), \iqtpi and \iqtpd achieve comparable coverage, but \iqtpd is prone to constructing much larger prediction intervals than \iqtpi. Smaller and midsized learning rates ($\constlr=0.01, 0.1$) regulate the size of \iqtpd intervals, leading to tight prediction intervals which maintain desired coverage levels.

\vspace{-0.8em}
\section{ConformalDAgger: A Calibrated Approach to Asking for Expert Feedback}
\vspace{-0.8em}
\label{sec:conformaldagger}

\begin{figure}[t]
    \centering
    \includegraphics[width=0.95\linewidth]{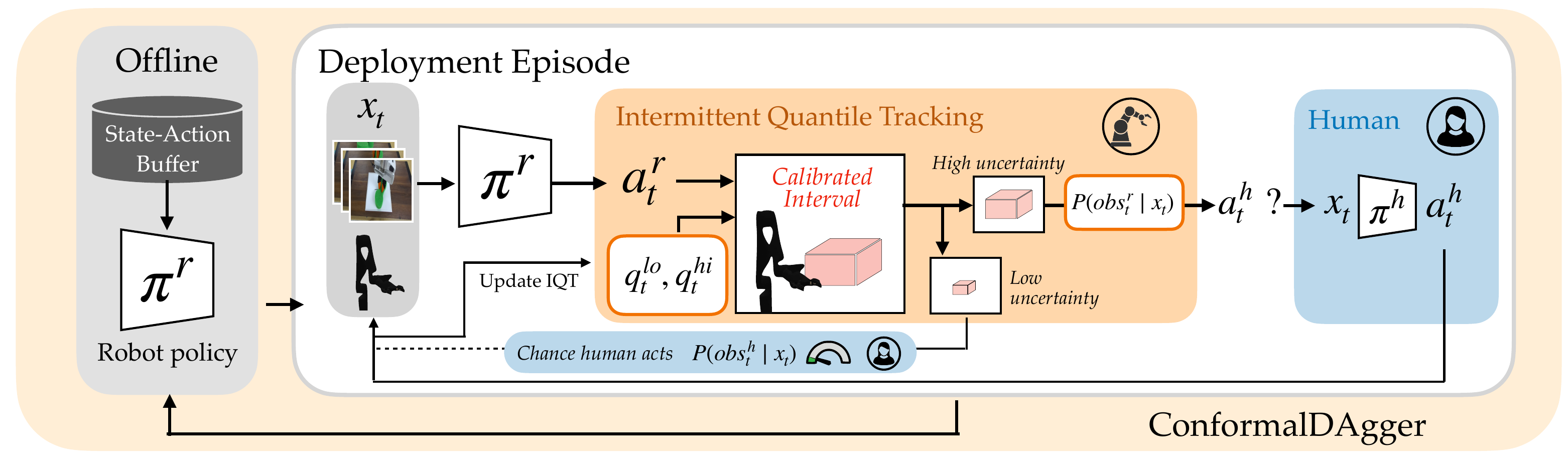}
    \caption{\textbf{ConformalDAgger Framework.} After obtaining an initial learner policy $\nov$ (left), ConformalDAgger calibrates uncertainty \textit{during} the interactive deployment episode with the expert via intermittent quantile tracking (center). When the size of the uncertainty intervals is high, the robot actively queries the user for feedback. When uncertainty is low, the robot executes its predicted action and the human may independently intervene with some low probability. After deployment episode ends, the data is aggregated and the learner retrained (arrow from right to left).}
    \label{fig:daggerloop}
    \vspace{-2em}
\end{figure}



Since intermittent quantile tracking enables us to rigorously quantify the learner's uncertainty despite the fact that the labels are only revealed intermittently (i.e., when the expert intervenes), we can develop ConformalDAgger: a new way for the robot to tradeoff between acting autonomously and strategically asking for help when uncertainty increases. 

\vspace{-0.2em}
\para{Setup}
ConformalDAgger treats the robot's policy as the base model $\hat{f} := \nov$ on which we perform intermittent quantile tracking. 
The initial novice policy $\nov_0: \XX \rightarrow \AA$ is trained on initial demonstration data $\DD_0$ of task $\TT$ performed by the expert $\expert_0$. 
Here, $\XX$ represents the policy's inputs (e.g., image observations, proprioception), $\YY \equiv \AA$ are the labels representing the robot's actions (e.g., future end-effector positions). 
During deployment, the robot policy generates a sequence of input-predicted action pairs $(x_t, \actnov_t) \in \XX \times \YY$ for $t=1,2,...$, that are temporally correlated. 
Relatedly, for each input $x_t$ the learner observes, there is a corresponding expert action $(x_t, \actexp_t) \in \XX \times \YY$ that is the ground-truth label we seek to cover via our IQT prediction intervals $C_t(x_t)$.

\vspace{-0.2em}
\para{Observation Model}
A key component of IQT is the observation likelihood model, $p_t = \PP(\obst = 1 \mid x_t)$. A nice byproduct of this model is that we can naturally derive a feedback model that is simultaneously \textit{human-} and \textit{robot-}gated by decomposing the observation model into the combination of both gating functions. 
Intuitively, the likelihood of observing the expert feedback at $t$ is given by the probability that the human chooses to give feedback \textit{or} the robot asks the expert for feedback. Let $\obs^h_t \in \{0,1\}$ be a random variable representing observing a human-gated feedback (if $\obs^h_t=1$) where the expert initiates providing an action label $a^h_t = \expert(x_t)$ for $x_t$.
Let $\obs^r_t \in \{0,1\}$ be a random variable representing robot-gated feedback, where the robot asks the expert for a label if $\obs^r_t=1$. 
We assume that if the robot elects to ask a question, the human will respond with probability 1. 
Our observation model takes the form:  $ p_t := \PP(\obs_t=1\mid x_t) =  \PP(\obs^h_t=1 \lor \obs^r_t=1\mid x_t) = \PP(\obs^h_t=1\mid x_t)  + \PP(\obs^r_t=1\mid x_t) -  \PP(\obs^h_t=1 \mid x_t) \PP(\obs^r_t=1 \mid x_t).$
While our approach is not prescriptive about these models, in our experiments 
we model $\PP(\obs^h_t=1\mid x_t) = c$ as a small constant to represent infrequent human interventions. We describe the robot-gated likelihood model $\PP(\obs^r_t=1\mid x_t)$ below, informed by our uncertainty estimates. 


\setlength{\textfloatsep}{4pt}
\begin{algorithm}[t!]
\caption{ConformalDAgger \textit{(changes from DAgger \citep{ross2011reduction} highlighted)}}
\label{alg:conformaldagger}
\begin{algorithmic}[1]
\State Collect initial demonstration data $\DD_0$ from expert $\expert_0$ and train initial learner policy $\nov_0$.

\For {interactive deployment episode $i = 0: M$}
    \textcolor{orange}{\State Initialize $q^{lo}_{0}, q^{hi}_{0}$.
    \For {deployment timestep $t = 1: H$} 
        \State Get predicted action label: $\actnov_t \leftarrow \nov_i(x_t)$
        \State Construct calibrated uncertainty interval: $C_{t}(x_{t}) = [\actnov_{t} - q^{lo}_{t}, \actnov_{t} + q^{hi}_{t}]$.
        \State Compute robot query likelihood (e.g. $\PP(\obs^r_t\mid o_t;\tau) = 1$ if $u(x_t; \nov) > \tau$, else 0.)
        \If{robot queries (w.p. $\PP(\obs^r_t\mid o_t;\tau)$) or human intervenes (w.p. $\PP(\obs^h_t\mid o_t)$)} 
            \State IQT update: $q^{lo,hi}_{t+1}  \leftarrow q^{lo,hi}_t + \frac{\gamma_t}{p_t} (\errt - \alpha)\obst$ ~w/ expert action $\actexp_t \leftarrow \expert_i(x_t)$.
        \Else
            \State IQT update: $q^{lo,hi}_{t+1}  \leftarrow q^{lo,hi}_t$ 
        \EndIf
    \EndFor}
    \State Aggregated dataset with observed state and expert action pairs: $\DD_{i+1}  \leftarrow \DD_{i} \cup \{(x, \expert_i(x))\}$
    \State Retrain learner: $\nov_{i+1} \leftarrow \argmin_{\pi} \mathcal{L} (\DD_{i+1}) $
\EndFor

\end{algorithmic}
\vspace{-0.1em}
\end{algorithm}

\vspace{-0.2em}
\para{Quantifying Uncertainty: IQT}
In the interactive IL setting, IQT begins with initial upper and lower quantiles in the action (i.e. ``label'') space, $q^{lo}_0, q^{hi}_0 \in \AA$. 
The nonconformity score is a residual on the predicted ($\actnov$) versus expert ($\actexp$) action. 
Let $s^{lo}_t(\actnov_t, \actexp_t) = \actnov_t - \actexp_t$ be the lower residual (referred to as $s^{lo}_t$ for brevity), and $s^{hi}_t(\actnov_t, \actexp_t) = \actexp_t - \actnov_t$ be the upper residual (referred to as $s^{hi}_t$). 
If the expert action $\actexp_t$ is observed, we compute the upper and lower miscoverage of $q^{lo}_t, q^{hi}_t$ as two indicator vectors: $\errt^{lo} = (s^{lo}_t < q^{lo}_t)$ and $\errt^{hi} = (s^{hi}_t < q^{hi}_t)$. 
IQT then updates the quantile estimates online to obtain $q^{lo}_{t+1}, q^{hi}_{t+1}$ via the update rule from ~\eqref{eq:q_update_iqt}. 
At the next timestep, the adjusted prediction interval is constructed with $C_{t+1}(x_{t+1}) = [\actnov_{t+1} - q^{lo}_{t+1}, \actnov_{t+1} + q^{hi}_{t+1}]$.
If the expert action is not observed, $\obst=0$, then $q^{lo}_{t+1} =  q^{lo}_{t}$ and $q^{hi}_{t+1} =  q^{hi}_{t}$ and the prediction interval size remains the same. 
Note that although $\errt$ is not known in the case where the expert does not provide an action label, IQT does not require it; IQT simply makes no change to the quantile estimate. 
\edit{In our simulated and hardware experiments, our action space, $q^{lo}_0$, and $q^{hi}_0$, are vector-valued. Our experiments instantiate IQT for continuous vectors, but the approach extends to discrete-valued action spaces.}

\vspace{-0.2em}
\para{Leveraging Uncertainty: Asking for Help}
Finally, the calibrated intervals $C_{t+1}(x_{t+1})$ constructed by IQT enable us to design a new robot-gated feedback mechanism. 
Specifically, the model $\PP(\obs^r_t \mid x_t)$ is informed by the calibrated interval size, $u(x_{t+1}) := || C_{t+1}(x_{t+1})||_2$. 
In simulation, we use  
$\PP(\obs^r_t \mid x_t) = \sigma(\beta [u(x_t; \nov) - \tau])$
where $\tau$ \edit{acts as an} uncertainty threshold \edit{on $u(x_t; \nov)$}, $\beta$ is a temperature hyperparameter, and $\sigma$ is the sigmoid function. 
In hardware experiments, we \edit{set $\tau$ as} a hard threshold on \edit{on $u(x_t; \nov)$} above which $\PP(\obs^r_t\mid x_t) = 1$, and below which $\PP(\obs^r_t\mid x_t)=0$. 


\vspace{-0.2em}
\para{ConformalDAgger Algorithm}
We summarize ConformalDAgger in Algorithm \ref{alg:conformaldagger}, \edit{highlighting the difference from traditional DAgger \citep{ross2011reduction}}. 
The learner interacts with the intermittent expert during $M$ interactive episodes. At the start of each, $q^{lo}, q^{hi}$ and are reset. Each episode lasts $H$ time steps and the intermittently-observed expert state-action pairs $(x_t, \actexp_t)$ are aggregated with the prior training data in a fixed size training buffer to form the updated dataset, $\DD_{i+1}$, which enables us to retrain the learner $\nov_{i+1}$ for the next deployment episode. 

\vspace{-0.8em}
\section{Simulated Interactive Imitation Learning Experiments}
\vspace{-0.8em}
\label{sec:experiments_dagger}
To evaluate ConformalDAgger,
we run a series of simulated experiments with access to an oracle expert. 
We ground our experiments in a simulated robot goal-reaching task (left, Figure~\ref{fig:simulated_expert_p01}) and study scenarios where the distribution shift occurs due to the expert's changing preference in goal location. 

\vspace{-0.8em}
\subsection{Experimental Setup}

\vspace{-0.8em}
\para{Task \& Initial Learner Policy}
The robot learns a neural network policy $\nov: \XX \rightarrow \AA$ to move a cup from a start to the expert's desired goal location. 
We model input $x \in \XX \subseteq \RR^{12}$ as the xyz position of the robot as well as its binary gripper state (open or closed) across the previous 3 timesteps. The labels are actions, $a \in \AA \subseteq \RR^4$, represented as next xyz position and gripper state. 
The robot always starts at an initial $x$ with the gripper state closed (holding a cup) and must keep holding the cup while moving it to the unknown goal location that the expert prefers; let $\{g_0, g_1\}$ be two such goals. 
Before the first interactive deployment ($i=0$), we simulate the expert as initially giving demonstrations placing the cup at $g_0$, where $\expert_{i=0}(x) = x + \omega \frac{g_0 - x}{\max_d (\mid g_0 - x\mid )} $, where $\max_d (\cdot)$ is a maximum over dimensions. The step size is regulated to be at most $\omega=0.01$.
The initial robot policy $\nov_{i=0}$ is trained on a dataset $\DD_0$ of 10 expert trajectories with synthetically injected noise drawn from $\N(1,0.5)$ for robustness as in \citep{laskey2017dart}. See Appendix~\ref{sec:extended_simulated_results} for policy implementation details.


\vspace{-0.2em}
\para{Simulated Expert Policies}
To induce controlled distribution shift, we study three expert policies (left, column of Figure~\ref{fig:simulated_expert_p01}). (1) \textit{Stationary}: the expert has a fixed goal, $g_0$, across all deployment episodes. (2) \textit{Shift}: the expert goal shifts from $g_0$ to $g_1$ at deployment episode $i=5$. 
For example, the expert may have decided that a different cup location is easier to reach. 
(3) \textit{Drift}: expert's goal slowly drifts from $g_0$ to $g_1$ over the course of deployment episodes (in Figure~\ref{fig:simulated_expert_p01} the drift from $g_0 \rightarrow g_{1a}$ occurs at episode $i=5$,  $g_{1a} \rightarrow g_{1b}$ occurs at episode $i=8$, and $g_{1b} \rightarrow g_{1}$ occurs at episode $i=11$).
For example, the expert may start with a conservative goal location  initially (e.g., a goal nearby) and incrementally move the cup closer and closer to their target goal that may be further out of reach.
\edit{(4) \textit{Environment Shift}: Under environment shift, the policy experiences covariate shift, where the deployment-time state distribution differs from the training distribution. Because the policy is trained on inputs which capture the robot’s position, we instantiated environment shift by changing the starting position of the robot in the first episode, maintaining this new position for all latter episodes.}


\vspace{-0.2em}
\para{Interactive Deployment Episodes \& Learner Re-training}
We consider $M = 15$ deployment episodes before re-training the learner. Each deployment episode has two interactive task executions. 
The task ends when the cup has reached the correct goal position, $g^*$, or when the maximum timesteps (100) have been reached. 
The expert answers queries with optimal actions under their current policy $a^h = \expert_i(x)$. 
Following DAgger \citep{ross2011reduction}, after each deployment episode, the state-action pairs where the expert provided action labels are aggregated into the training dataset, forming aggregated buffer $\DD_{i+1}$ for the next deployment episode $i+1$. 
We constrain the size of the replay buffer to 300 datapoints, dropping the old experiences.

\vspace{-0.2em}
\para{Methods}
We compare \cdagger to \ensemble \citep{menda2019ensembledagger}, \edit{\lazy \citep{hoque2021lazydagger}, and \safe \citep{zhang2016query}.} 
\edit{We performed a hyperparameter sensitivity analysis for all methods (see \ref{sec:hyperparam_analysis}) and 
selected thresholds such that all approaches ask infrequent questions in the first few interactive deployment episodes, where no shift has occurred.}
\cdagger uses an uncertainty threshold \edit{$\tau = 0.06$}, temperature $\beta=100$, lookback window $k=100$, $\constlr=0.6$, and initial $q^{lo,hi}_0 = 0.01$. 
The uncertainty threshold $\tau$ is heuristically tuned to ask few, but infrequent questions in the first interactive deployment episode. 
\ensemble queries an expert online when there is high action prediction variance across an ensemble of learner policies, and when a safety classifier detects dissimilarity between expert and robot actions. 
We use 3 ensemble members and an uncertainty threshold of $\tau = 0.06$ for the ensemble disagreement and a safety classifier threshold of $s = 0.03$. 
\edit{\lazy uses $s=0.03$ to begin human intervention, and only switches back to autonomous mode when the deviation between the learner’s prediction and expert's action are below a context-switching threshold, $0.1 * s$. To make \safe's number of initial interventions comparable, we decrease the safety classifier threshold to $s=0.01$.
}


\vspace{-0.2em}
\para{Metrics} We measure the quality of our uncertainty quantification via the \textbf{miscoverage rate} and human effort via the \textbf{intervention percentage} of the deployment trajectory. We compute the miscoverage for EnsembleDAgger using three times the standard deviation as the prediction interval. We measure the quality of the learned policy via two metrics. \textbf{Decision deviation} simulates the learner's behavior under its current policy and queries the expert at each learner state to obtain an expert action label. We measure the average L2 distance between the predicted and expert action. \textbf{Trajectory deviation} forward simulates both the expert and the learner acting independently under their policies, starting from the same initial $x$. We compare the L2 distance between the trajectories. 

\begin{figure}[t!]
    \centering
    \includegraphics[width=0.85\linewidth]{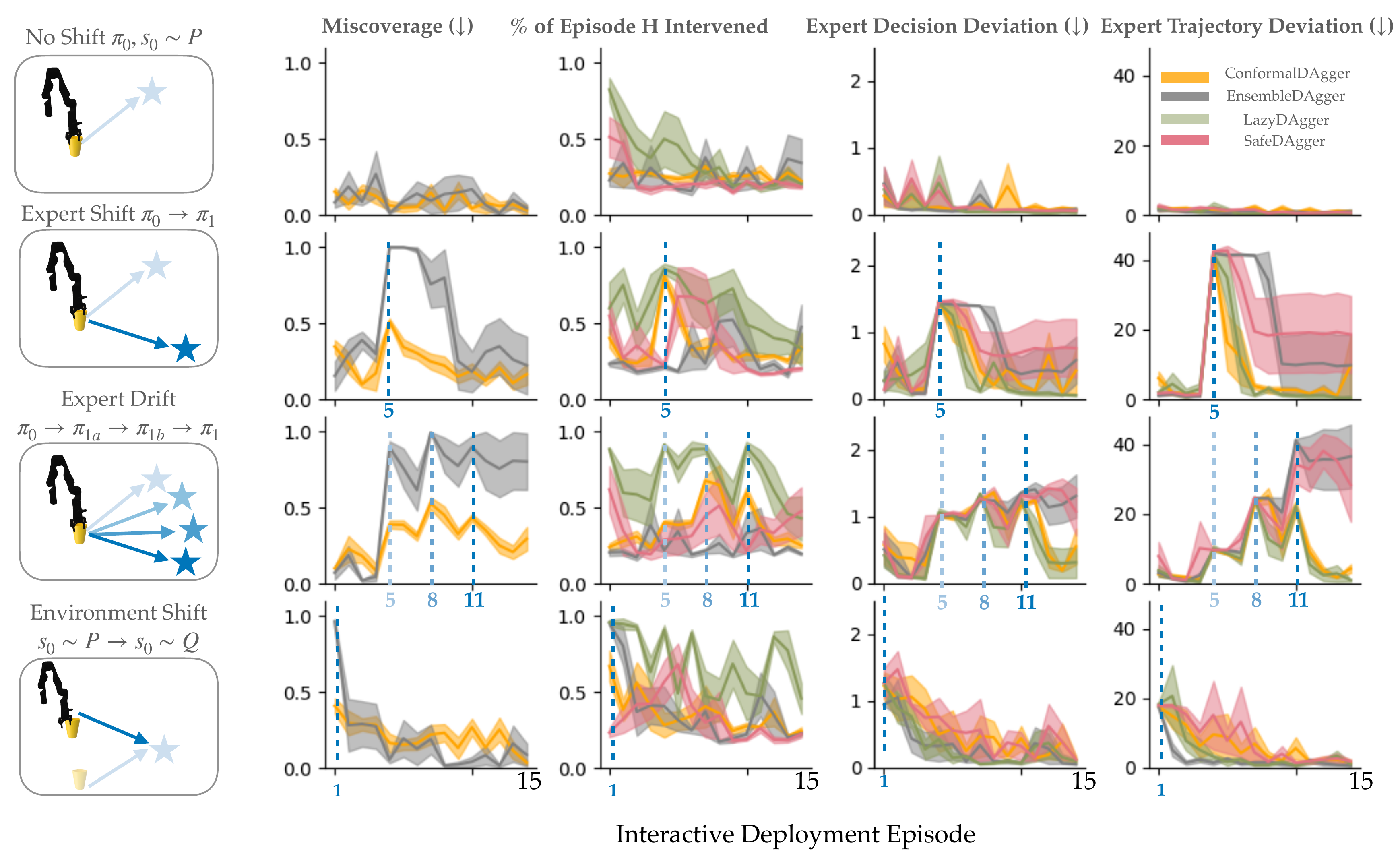}
    \caption{\textbf{Simulated Robot Results.} When the expert shifts or drifts (middle rows), \cdagger increases the number of requests for expert feedback compared to \ensemble \edit{ and \safe, but has less interventions than \lazy}. 
    This decreases miscoverage and expert deviation after re-training. 
    With a stationary expert (top) \edit{or environment shift (bottom), both \ensemble and \cdagger perform similarly. Shading is std. error across 5 seeds.} }
    \label{fig:simulated_expert_p01}
\end{figure}

\vspace{-0.8em}
\subsection{Experimental Results}
\vspace{-0.8em}
We focus results on infrequent human-gated feedback where  $\PP(\obs^h_t=1 \mid x_t) = 0.2, \forall t$. In Appendix \ref{sec:extended_simulated_results}, we present experiments with frequent (=$0.9$) and partial (=$0.5$) feedback. Irrespective of the human-gated likelihood, the learner can always \textit{actively} ask for feedback based on $\PP(\obs^r_t \mid x_t)$. 


\vspace{-0.2em}
\textbf{Takeaway 1: When the expert \textit{shifts}, \cdagger asks for more help immediately compared to \ensemble \edit{and \safe}, enabling the algorithm to more quickly align to the expert \edit{with fewer human interventions than \lazy.}}
Consider deployment episode $i=5$ where the expert shifts from goal 0 to goal 1 (center row, Figure \ref{fig:simulated_expert_p01}). ConformalDAgger immediately increases the number of expert feedback requests: before shift, the expert intervened $\sim$20\% of the time but in the 5th episode they intervene $\sim$60\% of the time. 
In contrast, EnsembleDAgger remains close to $\sim$30\%. 
Relatedly, ConformalDAgger has consistently lower miscoverage rate (max = 0.4 at shift; converges to 0.2) compared to the baseline (max = 1.0 at shift, converges to 0.4). Due to the extra solicited feedback, ConformalDAgger's ultimate decision and trajectory deviations are minimized in the subsequent retrained policies. 

\vspace{-0.2em}
\textbf{Takeaway 2: \cdagger automatically asks for more help each time the expert \textit{drifts}.} In the third row of Figure \ref{fig:simulated_expert_p01}, we see ConformalDAgger maintain a similar level of queries as EnsembleDAgger during the first shift (at episode 5) and last shift (episode 11) but increases feedback during the intermediate shift (at episode 8). Despite these similarities, EnsembleDAgger's miscoverage rate is consistently higher than ConformalDAgger's and EnsembleDAgger's re-trained policy on average does not adapt as quickly (with higher expert decision and trajectory deviation). We hypothesize that this is because ConformalDAgger asks frequent questions consistently across all seeds compared to EnsembleDAgger (i.e. lower variance in Fig \ref{fig:simulated_expert_p01}). 


\vspace{-0.2em}
\textbf{Takeaway 3: With a \textit{stationary} expert, \cdagger and \edit{baselines} are similar.} 
We find that without distribution shift (top row, Figure~\ref{fig:simulated_expert_p01}), all methods ask for minimal help (staying near the 20\% human-gated probability), with \lazy asking slightly more questions, and achieve similar performance in alignment to the expert. 

\vspace{-0.2em}
\edit{\textbf{Takeaway 4: Under \textit{environment shift}, \cdagger is comparable to baselines.} 
In the fourth row of Figure~\ref{fig:simulated_expert_p01}, 
we see that \ensemble and \cdagger balance intervention frequency with policy learning quality marginally better than \lazy or \safe. }

\vspace{-1.1em}
\section{Hardware Experiments}
\vspace{-1.1em}
\label{sec:results}
Finally, we deployed \cdagger in hardware on a 7 degree-of-freedom robotic manipulator that uses a state-of-the-art Diffusion Policy \citep{chi2023diffusion} trained via IL to perform a sponging task (Figure~\ref{fig:real_robot_intuition}). 
The goal of our hardware experiments is to demonstrate how our approach can scale to a high-dimensional, real-world policy and understand how ConformalDAgger enables the robot learner to query a real human teleoperator. 
Our goal is to train the robot to perform a real-world cleaning task where it wipes up a line drawn by an Expo marker on a whiteboard with sponge. 



\vspace{-0.1em}
\para{Human Expert}
The human expert teleoperates the robot via a Meta Quest 3 remote controller. They initially provide \edit{50} demonstrated trajectories moving in a straight-line path along the marker line to construct the initial training dataset, $\DD_0$. 
We model human-gated interventions with probability $\PP(\obs^h_t=1 \mid x_t)=0.2, ~\forall t$. To control this rate, our interface queries the user with 20\% probability at the any timepoint along deployment to simulate that consistent independent intervention rate.  
We induce a potential distribution shift during interactive deployment episodes by changing the expert's wiping strategy from a straight line to a zig-zag pattern. 

\vspace{-0.2em}
\para{Learner Policy} 
We represent the robot's policy as a CNN-based diffusion policy \citep{chi2023diffusion}. 
The policy predicts $a^r \in \AA \equiv \YY$ as 16 future actions, where each action is an end-effector position and quaternion orientation. 
The inputs $x \in \XX$ are the current and previous image observations from the wrist and third-person camera. 
Image observations are encoded using a ResNet-18 visual encoder (trained end-to-end with the diffusion policy) and the action-generating process is conditioned on encoded observation features with FiLM \citep{perez2018film}. 
The initial policy $\nov_0$ is trained for \edit{50K} iterations (training parameters in Appendix Sec \ref{subsec:real_learner_training}). The robot is deployed to interactively execute the task \edit{$25$} times after which its policy is retrained. This is an interactive deployment where the robot queries the user via our ConformalDAgger algorithm, or the user can independently intervene intermittently. 
During expert feedback, the human teleoperates the robot for a sequence of 16 timesteps, after which, if the robot doesn't ask for feedback, control is handed back over to the robot. 
After the interactive deployment, the initial learner policy is fine-tuned for an additional 35K iterations. We reset the learning-rate schedule, aggregate datapoints at which the expert gave feedback, and cap the training buffer size at 6k most recent input-action pairs. 


\vspace{-0.2em}
\para{ConformalDAgger Hyperparameters}
We apply IQT on only the predicted end-effector position, because we reason about nonconformity via the signed residual in Euclidean space. 
\edit{For quaternion rotation, we calibrate only the positional output to avoid the complexity of multiple IQT processes for positional and rotational nonconformities.}

We initialize $q^{lo}_0, q^{hi}_0 = 0.01$, our desired coverage level is $\alpha = 0.1$, and $\lrt = 0.15 \hat{B}_t$, where the lookback window for $\hat{B}_t$ is 20. 
We use a simple model for $\PP(\obs^r \mid x_t)=1$ that asks for help if the L2 norm of the uncertainty interval $C_t(x_t)$ exceeds threshold $\tau = 0.07$. 

\begin{figure}[t!]
    \centering
    \includegraphics[width=0.9\linewidth]{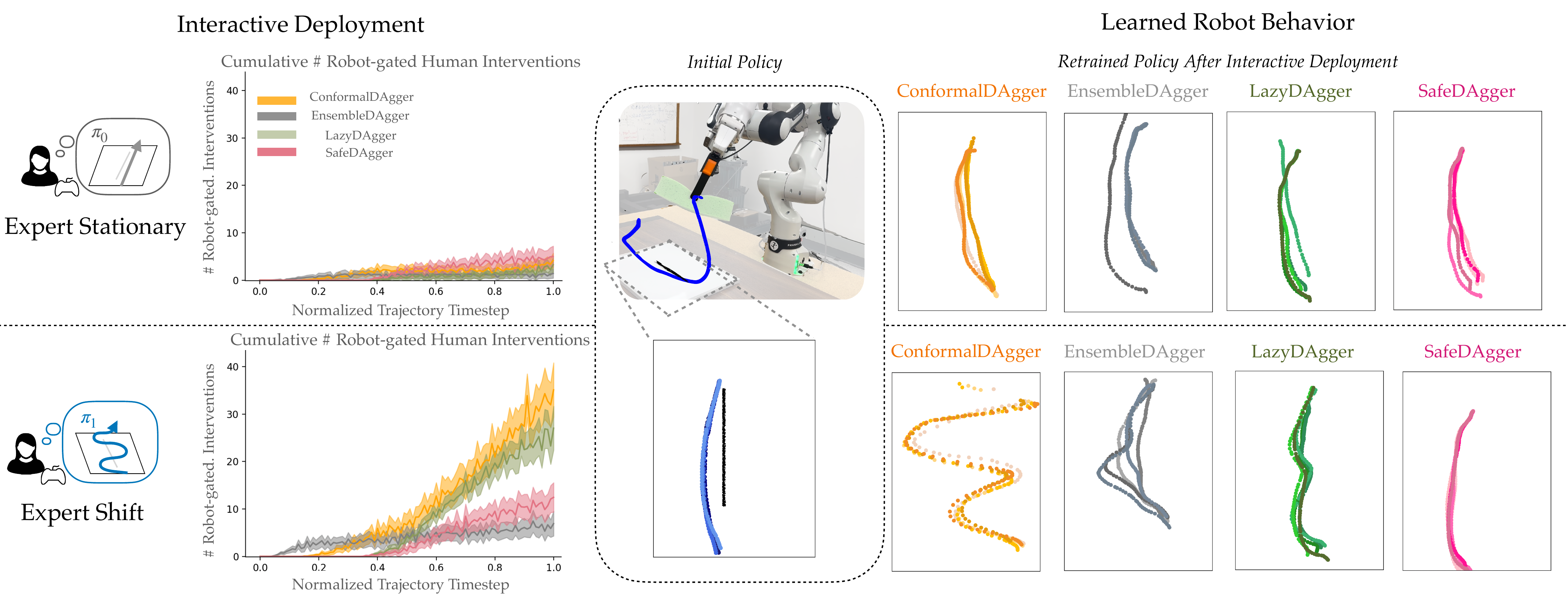}
    \vspace{-0.5em}
    \caption{\edit{\textbf{Hardware Results.} (bottom, left) Conformal enables the robot to significantly increase the number of questions when the expert shifts their strategy while Safe and Ensemble only increase slightly. (bottom, right) Qualitatively, Conformal learns a higher quality zigzag strategy. Lazy exhibits a similar an intervention profile to Conformal, but because the algorithm doesn't ask for help consistently in each episode, it does not receive enough additional data to learn the shifted policy.}}
    \label{fig:real_rollouts}
\end{figure}

\vspace{-0.2em}
\para{Baseline}
Because the diffusion policy implicitly represents the action distribution present in the training data, we sample from the policy N=3 times and evaluate variance over the predicted action in order to capture the notion of epistemic uncertainty. We consider this an \textit{implicit} \ensemble: \edit{we use the implicit ensemble approach from \citep{wolleb2022diffusion} by sampling multiple times from the same diffusion model}. \edit{The safety classifier is a fully connected network that takes as input the ResNet-18 embeddings of the camera images concatenated with robot proprioceptive state. We set the classifier threshold $s$ to be the mean of the distribution of expert-learner divergence over the training datapoints.} 
\ensemble's uncertainty threshold is \edit{0.005, at approximately the mean of the distribution of ensembles variances (over position and orientation predicted by the policy) over the training data}, tuned such that the learner asks occasional but not excessive questions. \edit{We also compared with \lazy (threshold of $0.25 * s$, informed by \cite{hoque2021thriftydagger}), and \safe. All models use the same initial learner policy.}

\vspace{-0.2em}
\para{Results}
The left of Figure~\ref{fig:real_rollouts} shows quantitative results for both methods. 
\cdagger increases the cumulative number of robot-requested interventions up to about 35 timesteps of human feedback when interacting with the \textit{shifted} expert (who switches to a zig-zag pattern) compared to the \textit{stationary} expert (who always wipes in a straight line) \textit{which does not induce as frequent of questions}. 
To better understand this, we analyzed Figure~\ref{fig:real_robot_intuition}, which shows \cdagger engaging with each type of expert in one interactive deployment episode. Here, we see how ConformalDAgger has uncertainty when interacting with the stationary expert (due to inherent noise in the VR-teleoperated intermittent human interventions) but the robot's uncertainty never rises above the threshold the robot \textit{initiates} expert feedback. This is in stark contrast to the expert shift scenario (right of Figure~\ref{fig:real_robot_intuition}) where half-way through the deployment episode the uncertainty intervals become large enough to trigger robot asking for help. 
On the other hand, \edit{with the shifted expert, \ensemble and \safe do not request as many interventions as \lazy and \cdagger, (Figure~\ref{fig:real_rollouts}). }
\edit{When \lazy asks for help the queries follow a similar profile as \cdagger, but it does not reliably ask every episode, leading to less data for learning the shifted strategy.}
We hypothesize that this may be because the training demonstrations have low variance while the sponge is in contact with the whiteboard, but higher variance as the robot approaches the table. 
Because \ensemble's and \safe's uncertainty is uncalibrated to the deployment-time expert's data, it is unable to identify the need for additional feedback when it reaches the whiteboard. Qualitative results are shown on the right of Figure \ref{fig:real_rollouts}. As expected, all approaches remain aligned with the \textit{stationary} expert when retrained. 
However, under the \textit{shifted} expert, rollouts from the retrained \cdagger policy exhibit a more distinct zig-zag pattern, compared to \ensemble and \safe, likely because the approaches did not ask for as much help. 
\vspace{-1.1em}
\section{Conclusion}
\vspace{-1.1em}
\label{sec:conclusion}
We first extend uncertainty quantification via online conformal prediction to handle intermittent labels, such as those observed in interactive imitation learning. 
We then propose ConformalDAgger, a unification of our online conformal prediction algorithm with interactive imitation learning. 
Our approach provides asymptotic coverage guarantees for deployed end-to-end policies, uses the calibrated uncertainty measure to detect expert distribution shifts and actively query for more feedback, and empirically enables the robot learner update its policy to better align with the expert. 



\subsubsection*{Acknowledgments}
The authors would like to thank Gokul Swamy for insightful conversations and the detailed review, Yilin Wu for help with diffusion policy and robot hardware setup. MZ is supported by an NDSEG fellowship. 

\bibliography{iclr_references}

\begin{thebibliography}{52}
\providecommand{\natexlab}[1]{#1}
\providecommand{\url}[1]{\texttt{#1}}
\expandafter\ifx\csname urlstyle\endcsname\relax
  \providecommand{\doi}[1]{doi: #1}\else
  \providecommand{\doi}{doi: \begingroup \urlstyle{rm}\Url}\fi

\bibitem[Angelopoulos \& Bates(2023)Angelopoulos and Bates]{angelopoulos2021gentle}
Anastasios Angelopoulos and Stephen Bates.
\newblock Conformal prediction: A gentle introduction.
\newblock \emph{Foundations and Trends in Machine Learning}, 16\penalty0 (4):\penalty0 494--591, 2023.

\bibitem[Angelopoulos et~al.(2023)Angelopoulos, Candes, and Tibshirani]{angelopoulos2024conformal}
Anastasios Angelopoulos, Emmanuel Candes, and Ryan~J Tibshirani.
\newblock Conformal pid control for time series prediction.
\newblock \emph{Advances in neural information processing systems}, 36:\penalty0 23047--23074, 2023.

\bibitem[Angelopoulos et~al.(2024)Angelopoulos, Barber, and Bates]{angelopoulos2024online}
Anastasios~N Angelopoulos, Rina~Foygel Barber, and Stephen Bates.
\newblock Online conformal prediction with decaying step sizes.
\newblock In \emph{Proceedings of the 41st International Conference on Machine Learning}, 2024.

\bibitem[Argall et~al.(2009)Argall, Chernova, Veloso, and Browning]{argall2009survey}
Brenna~D Argall, Sonia Chernova, Manuela Veloso, and Brett Browning.
\newblock A survey of robot learning from demonstration.
\newblock \emph{Robotics and autonomous systems}, 57\penalty0 (5):\penalty0 469--483, 2009.

\bibitem[Assimakopoulos \& Nikolopoulos(2000)Assimakopoulos and Nikolopoulos]{assimakopoulos2000theta}
Vassilis Assimakopoulos and Konstantinos Nikolopoulos.
\newblock The theta model: a decomposition approach to forecasting.
\newblock \emph{International Journal of Forecasting}, 16\penalty0 (4):\penalty0 521--530, 2000.

\bibitem[Bhatnagar et~al.(2023)Bhatnagar, Wang, Xiong, and Bai]{bhatnagar2023improved}
Aadyot Bhatnagar, Huan Wang, Caiming Xiong, and Yu~Bai.
\newblock Improved online conformal prediction via strongly adaptive online learning.
\newblock In \emph{International Conference on Machine Learning}, pp.\  2337--2363. PMLR, 2023.

\bibitem[Celemin et~al.(2022)Celemin, P{\'e}rez-Dattari, Chisari, Franzese, de~Souza~Rosa, Prakash, Ajanovi{\'c}, Ferraz, Valada, Kober, et~al.]{celemin2022interactive}
Carlos Celemin, Rodrigo P{\'e}rez-Dattari, Eugenio Chisari, Giovanni Franzese, Leandro de~Souza~Rosa, Ravi Prakash, Zlatan Ajanovi{\'c}, Marta Ferraz, Abhinav Valada, Jens Kober, et~al.
\newblock Interactive imitation learning in robotics: A survey.
\newblock \emph{Foundations and Trends{\textregistered} in Robotics}, 10\penalty0 (1-2):\penalty0 1--197, 2022.

\bibitem[Chang et~al.(2021)Chang, Uehara, Sreenivas, Kidambi, and Sun]{chang2021mitigating}
Jonathan Chang, Masatoshi Uehara, Dhruv Sreenivas, Rahul Kidambi, and Wen Sun.
\newblock Mitigating covariate shift in imitation learning via offline data with partial coverage.
\newblock \emph{Advances in Neural Information Processing Systems}, 34:\penalty0 965--979, 2021.

\bibitem[Chen et~al.(2021)Chen, Rosolia, Fan, Ames, and Murray]{chen2021reactive}
Yuxiao Chen, Ugo Rosolia, Chuchu Fan, Aaron Ames, and Richard Murray.
\newblock Reactive motion planning with probabilistic safety guarantees.
\newblock In \emph{Conference on Robot Learning}, pp.\  1958--1970. PMLR, 2021.

\bibitem[Chi et~al.(2023)Chi, Feng, Du, Xu, Cousineau, Burchfiel, and Song]{chi2023diffusion}
Cheng Chi, Siyuan Feng, Yilun Du, Zhenjia Xu, Eric Cousineau, Benjamin Burchfiel, and Shuran Song.
\newblock Diffusion policy: Visuomotor policy learning via action diffusion.
\newblock In \emph{Robotics: Science and Systems (RSS)}, 2023.

\bibitem[Cui et~al.(2019)Cui, Isele, Niekum, and Fujimura]{cui2019uncertainty}
Yuchen Cui, David Isele, Scott Niekum, and Kikuo Fujimura.
\newblock Uncertainty-aware data aggregation for deep imitation learning.
\newblock In \emph{2019 International Conference on Robotics and Automation (ICRA)}, pp.\  761--767. IEEE, 2019.

\bibitem[Dietterich \& Hostetler(2022)Dietterich and Hostetler]{dietterich2022conformal}
Thomas~G Dietterich and Jesse Hostetler.
\newblock Conformal prediction intervals for {M}arkov decision process trajectories.
\newblock \emph{arXiv preprint arXiv:2206.04860}, 2022.

\bibitem[Dixit et~al.(2023)Dixit, Lindemann, Wei, Cleaveland, Pappas, and Burdick]{dixit2023adaptive}
Anushri Dixit, Lars Lindemann, Skylar~X Wei, Matthew Cleaveland, George~J Pappas, and Joel~W Burdick.
\newblock Adaptive conformal prediction for motion planning among dynamic agents.
\newblock In \emph{Learning for Dynamics and Control Conference}, pp.\  300--314. PMLR, 2023.

\bibitem[Gibbs \& Candes(2021)Gibbs and Candes]{gibbs2021adaptive}
Isaac Gibbs and Emmanuel Candes.
\newblock Adaptive conformal inference under distribution shift.
\newblock \emph{Advances in Neural Information Processing Systems}, 34:\penalty0 1660--1672, 2021.

\bibitem[Gibbs \& Cand{\`e}s(2024)Gibbs and Cand{\`e}s]{gibbs2024conformal}
Isaac Gibbs and Emmanuel~J Cand{\`e}s.
\newblock Conformal inference for online prediction with arbitrary distribution shifts.
\newblock \emph{Journal of Machine Learning Research}, 25\penalty0 (162):\penalty0 1--36, 2024.

\bibitem[Harries(1999)]{harries1999splice}
Michael Harries.
\newblock Splice-2 comparative evaluation: Electricity pricing.
\newblock 1999.

\bibitem[Herzen et~al.(2022)Herzen, L{\"a}ssig, Piazzetta, Neuer, Tafti, Raille, Van~Pottelbergh, Pasieka, Skrodzki, and Huguenin]{herzen2022darts}
Julien Herzen, Francesco L{\"a}ssig, Samuele~Giuliano Piazzetta, Thomas Neuer, L{\'e}o Tafti, Guillaume Raille, Tomas Van~Pottelbergh, Marek Pasieka, Andrzej Skrodzki, and Nicolas Huguenin.
\newblock Darts: User-friendly modern machine learning for time series.
\newblock \emph{Journal of Machine Learning Research}, 23\penalty0 (124):\penalty0 1--6, 2022.

\bibitem[Hong et~al.(2024)Hong, Levine, and Dragan]{hong2024learning}
Joey Hong, Sergey Levine, and Anca Dragan.
\newblock Learning to influence human behavior with offline reinforcement learning.
\newblock \emph{Advances in Neural Information Processing Systems}, 36, 2024.

\bibitem[Hoque et~al.(2021{\natexlab{a}})Hoque, Balakrishna, Novoseller, Wilcox, Brown, and Goldberg]{hoque2021thriftydagger}
Ryan Hoque, Ashwin Balakrishna, Ellen Novoseller, Albert Wilcox, Daniel~S Brown, and Ken Goldberg.
\newblock Thriftydagger: Budget-aware novelty and risk gating for interactive imitation learning.
\newblock \emph{arXiv preprint arXiv:2109.08273}, 2021{\natexlab{a}}.

\bibitem[Hoque et~al.(2021{\natexlab{b}})Hoque, Balakrishna, Putterman, Luo, Brown, Seita, Thananjeyan, Novoseller, and Goldberg]{hoque2021lazydagger}
Ryan Hoque, Ashwin Balakrishna, Carl Putterman, Michael Luo, Daniel~S Brown, Daniel Seita, Brijen Thananjeyan, Ellen Novoseller, and Ken Goldberg.
\newblock Lazydagger: Reducing context switching in interactive imitation learning.
\newblock In \emph{2021 IEEE 17th international conference on automation science and engineering (case)}, pp.\  502--509. IEEE, 2021{\natexlab{b}}.

\bibitem[Hoque et~al.(2023)Hoque, Chen, Sharma, Dharmarajan, Thananjeyan, Abbeel, and Goldberg]{hoque2023fleet}
Ryan Hoque, Lawrence~Yunliang Chen, Satvik Sharma, Karthik Dharmarajan, Brijen Thananjeyan, Pieter Abbeel, and Ken Goldberg.
\newblock Fleet-dagger: Interactive robot fleet learning with scalable human supervision.
\newblock In \emph{Conference on Robot Learning}, pp.\  368--380. PMLR, 2023.

\bibitem[Jang et~al.(2022)Jang, Irpan, Khansari, Kappler, Ebert, Lynch, Levine, and Finn]{jang2022bc}
Eric Jang, Alex Irpan, Mohi Khansari, Daniel Kappler, Frederik Ebert, Corey Lynch, Sergey Levine, and Chelsea Finn.
\newblock Bc-z: Zero-shot task generalization with robotic imitation learning.
\newblock In \emph{Conference on Robot Learning}, pp.\  991--1002. PMLR, 2022.

\bibitem[Kelly et~al.(2019)Kelly, Sidrane, Driggs-Campbell, and Kochenderfer]{kelly2019hg}
Michael Kelly, Chelsea Sidrane, Katherine Driggs-Campbell, and Mykel~J Kochenderfer.
\newblock Hg-dagger: Interactive imitation learning with human experts.
\newblock In \emph{2019 International Conference on Robotics and Automation (ICRA)}, pp.\  8077--8083. IEEE, 2019.

\bibitem[Kim et~al.(2024)Kim, Pertsch, Karamcheti, Xiao, Balakrishna, Nair, Rafailov, Foster, Lam, Sanketi, et~al.]{kim2024openvla}
Moo~Jin Kim, Karl Pertsch, Siddharth Karamcheti, Ted Xiao, Ashwin Balakrishna, Suraj Nair, Rafael Rafailov, Ethan Foster, Grace Lam, Pannag Sanketi, et~al.
\newblock Open{VLA}: An open-source vision-language-action model.
\newblock \emph{arXiv preprint arXiv:2406.09246}, 2024.

\bibitem[Koenker \& Bassett~Jr(1978)Koenker and Bassett~Jr]{koenker1978regression}
Roger Koenker and Gilbert Bassett~Jr.
\newblock Regression quantiles.
\newblock \emph{Econometrica: journal of the Econometric Society}, pp.\  33--50, 1978.

\bibitem[Laskey et~al.(2017)Laskey, Lee, Fox, Dragan, and Goldberg]{laskey2017dart}
Michael Laskey, Jonathan Lee, Roy Fox, Anca Dragan, and Ken Goldberg.
\newblock Dart: Noise injection for robust imitation learning.
\newblock In \emph{Conference on robot learning}, pp.\  143--156. PMLR, 2017.

\bibitem[Levine et~al.(2016)Levine, Finn, Darrell, and Abbeel]{levine2016end}
Sergey Levine, Chelsea Finn, Trevor Darrell, and Pieter Abbeel.
\newblock End-to-end training of deep visuomotor policies.
\newblock \emph{Journal of Machine Learning Research}, 17\penalty0 (39):\penalty0 1--40, 2016.

\bibitem[Lidard et~al.(2024)Lidard, Pham, Bachman, Boateng, and Majumdar]{lidard2024risk}
Justin Lidard, Hang Pham, Ariel Bachman, Bryan Boateng, and Anirudha Majumdar.
\newblock Risk-calibrated human-robot interaction via set-valued intent prediction.
\newblock In \emph{Robotics: Science and Systems}, 2024.

\bibitem[Lin \& Bansal(2024)Lin and Bansal]{lin2024verification}
Albert Lin and Somil Bansal.
\newblock Verification of neural reachable tubes via scenario optimization and conformal prediction.
\newblock In \emph{6th Annual Learning for Dynamics \& Control Conference}, pp.\  719--731. PMLR, 2024.

\bibitem[Lindemann et~al.(2023)Lindemann, Cleaveland, Shim, and Pappas]{lindemann2023safe}
Lars Lindemann, Matthew Cleaveland, Gihyun Shim, and George~J Pappas.
\newblock Safe planning in dynamic environments using conformal prediction.
\newblock \emph{IEEE Robotics and Automation Letters}, 2023.

\bibitem[Luo et~al.(2022)Luo, Zhao, Kuck, Ivanovic, Savarese, Schmerling, and Pavone]{luo2022sample}
Rachel Luo, Shengjia Zhao, Jonathan Kuck, Boris Ivanovic, Silvio Savarese, Edward Schmerling, and Marco Pavone.
\newblock Sample-efficient safety assurances using conformal prediction.
\newblock In \emph{International Workshop on the Algorithmic Foundations of Robotics}, pp.\  149--169. Springer, 2022.

\bibitem[Menda et~al.(2019)Menda, Driggs-Campbell, and Kochenderfer]{menda2019ensembledagger}
Kunal Menda, Katherine Driggs-Campbell, and Mykel~J Kochenderfer.
\newblock Ensemble{DA}gger: A {B}ayesian approach to safe imitation learning.
\newblock In \emph{2019 IEEE/RSJ International Conference on Intelligent Robots and Systems (IROS)}, pp.\  5041--5048. IEEE, 2019.

\bibitem[Muthali et~al.(2023)Muthali, Shen, Deglurkar, Lim, Roelofs, Faust, and Tomlin]{muthali2023multi}
Anish Muthali, Haotian Shen, Sampada Deglurkar, Michael~H Lim, Rebecca Roelofs, Aleksandra Faust, and Claire Tomlin.
\newblock Multi-agent reachability calibration with conformal prediction.
\newblock \emph{arXiv preprint arXiv:2304.00432}, 2023.

\bibitem[Nguyen(2018)]{nguyen2018stock}
Cam Nguyen.
\newblock {S\&P} 500 stock data, 2018.
\newblock URL \url{https://www.kaggle.com/datasets/camnugent/sandp500}.

\bibitem[Nichol \& Dhariwal(2021)Nichol and Dhariwal]{nichol2021improved}
Alexander~Quinn Nichol and Prafulla Dhariwal.
\newblock Improved denoising diffusion probabilistic models.
\newblock In \emph{International Conference on Machine Learning}, pp.\  8162--8171. PMLR, 2021.

\bibitem[Perez et~al.(2018)Perez, Strub, De~Vries, Dumoulin, and Courville]{perez2018film}
Ethan Perez, Florian Strub, Harm De~Vries, Vincent Dumoulin, and Aaron Courville.
\newblock Film: Visual reasoning with a general conditioning layer.
\newblock In \emph{Proceedings of the AAAI conference on artificial intelligence}, volume~32, 2018.

\bibitem[Price \& Boutilier(2003)Price and Boutilier]{price2003accelerating}
Bob Price and Craig Boutilier.
\newblock Accelerating reinforcement learning through implicit imitation.
\newblock \emph{Journal of Artificial Intelligence Research}, 19:\penalty0 569--629, 2003.

\bibitem[Ren et~al.(2023)Ren, Dixit, Bodrova, Singh, Tu, Brown, Xu, Takayama, Xia, Varley, et~al.]{ren2023robots}
Allen~Z Ren, Anushri Dixit, Alexandra Bodrova, Sumeet Singh, Stephen Tu, Noah Brown, Pen Xu, Leila~Takayama Takayama, Fei Xia, Jake Varley, et~al.
\newblock Robots that ask for help: Uncertainty alignment for large language model planners.
\newblock In \emph{Conference on Robot Learning (CoRL)}. Proceedings of the Conference on Robot Learning (CoRL), 2023.

\bibitem[Romano et~al.(2019)Romano, Patterson, and Candes]{romano2019conformalized}
Yaniv Romano, Evan Patterson, and Emmanuel Candes.
\newblock Conformalized quantile regression.
\newblock \emph{Advances in neural information processing systems}, 32, 2019.

\bibitem[Romano et~al.(2020)Romano, Sesia, and Candes]{romano2020classification}
Yaniv Romano, Matteo Sesia, and Emmanuel Candes.
\newblock Classification with valid and adaptive coverage.
\newblock \emph{Advances in Neural Information Processing Systems}, 33:\penalty0 3581--3591, 2020.

\bibitem[Ross et~al.(2011)Ross, Gordon, and Bagnell]{ross2011reduction}
St{\'e}phane Ross, Geoffrey Gordon, and Drew Bagnell.
\newblock A reduction of imitation learning and structured prediction to no-regret online learning.
\newblock In \emph{Proceedings of the fourteenth international conference on artificial intelligence and statistics}, pp.\  627--635. JMLR Workshop and Conference Proceedings, 2011.

\bibitem[Sagheb et~al.(2023)Sagheb, Mun, Ahmadian, Christie, Bajcsy, Driggs-Campbell, and Losey]{sagheb2023towards}
Shahabedin Sagheb, Ye-Ji Mun, Neema Ahmadian, Benjamin~A Christie, Andrea Bajcsy, Katherine Driggs-Campbell, and Dylan~P Losey.
\newblock Towards robots that influence humans over long-term interaction.
\newblock In \emph{2023 IEEE International Conference on Robotics and Automation (ICRA)}, pp.\  7490--7496. IEEE, 2023.

\bibitem[Schaal(1996)]{schaal1996learning}
Stefan Schaal.
\newblock Learning from demonstration.
\newblock \emph{Advances in Neural Information Processing Systems}, 9, 1996.

\bibitem[Spencer et~al.(2021)Spencer, Choudhury, Venkatraman, Ziebart, and Bagnell]{spencer2021feedback}
Jonathan Spencer, Sanjiban Choudhury, Arun Venkatraman, Brian Ziebart, and J~Andrew Bagnell.
\newblock Feedback in imitation learning: The three regimes of covariate shift.
\newblock \emph{arXiv preprint arXiv:2102.02872}, 2021.

\bibitem[Swamy et~al.(2022)Swamy, Rajaraman, Peng, Choudhury, Bagnell, Wu, Jiao, and Ramchandran]{swamy2022minimax}
Gokul Swamy, Nived Rajaraman, Matt Peng, Sanjiban Choudhury, J~Bagnell, Steven~Z Wu, Jiantao Jiao, and Kannan Ramchandran.
\newblock Minimax optimal online imitation learning via replay estimation.
\newblock \emph{Advances in Neural Information Processing Systems}, 35:\penalty0 7077--7088, 2022.

\bibitem[Taufiq et~al.(2022)Taufiq, Ton, Cornish, Teh, and Doucet]{taufiq2022conformal}
Muhammad~Faaiz Taufiq, Jean-Francois Ton, Rob Cornish, Yee~Whye Teh, and Arnaud Doucet.
\newblock Conformal off-policy prediction in contextual bandits.
\newblock \emph{Advances in Neural Information Processing Systems}, 35:\penalty0 31512--31524, 2022.

\bibitem[Taylor \& Letham(2018)Taylor and Letham]{taylor2018forecasting}
Sean~J Taylor and Benjamin Letham.
\newblock Forecasting at scale.
\newblock \emph{The American Statistician}, 72\penalty0 (1):\penalty0 37--45, 2018.

\bibitem[Vaswani et~al.(2017)Vaswani, Shazeer, Parmar, Uszkoreit, Jones, Gomez, Kaiser, and Polosukhin]{vaswani2017attention}
Ashish Vaswani, Noam Shazeer, Niki Parmar, Jakob Uszkoreit, Llion Jones, Aidan~N Gomez, \L~ukasz Kaiser, and Illia Polosukhin.
\newblock Attention is all you need.
\newblock In \emph{Advances in Neural Information Processing Systems}, 2017.

\bibitem[Wolleb et~al.(2022)Wolleb, Sandk{\"u}hler, Bieder, Valmaggia, and Cattin]{wolleb2022diffusion}
Julia Wolleb, Robin Sandk{\"u}hler, Florentin Bieder, Philippe Valmaggia, and Philippe~C Cattin.
\newblock Diffusion models for implicit image segmentation ensembles.
\newblock In \emph{International Conference on Medical Imaging with Deep Learning}, pp.\  1336--1348. PMLR, 2022.

\bibitem[Xie et~al.(2021)Xie, Losey, Tolsma, Finn, and Sadigh]{xie2021learning}
Annie Xie, Dylan Losey, Ryan Tolsma, Chelsea Finn, and Dorsa Sadigh.
\newblock Learning latent representations to influence multi-agent interaction.
\newblock In \emph{Conference on robot learning}, pp.\  575--588. PMLR, 2021.

\bibitem[Zaffran et~al.(2022)Zaffran, F{\'e}ron, Goude, Josse, and Dieuleveut]{zaffran2022adaptive}
Margaux Zaffran, Olivier F{\'e}ron, Yannig Goude, Julie Josse, and Aymeric Dieuleveut.
\newblock Adaptive conformal predictions for time series.
\newblock In \emph{International Conference on Machine Learning}, pp.\  25834--25866. PMLR, 2022.

\bibitem[Zhang \& Cho(2017)Zhang and Cho]{zhang2016query}
Jiakai Zhang and Kyunghyun Cho.
\newblock Query-efficient imitation learning for end-to-end autonomous driving.
\newblock In \emph{Proceedings of the Thirty-First AAAI Conference on Artificial Intelligence}, 2017.

\end{thebibliography}
\bibliographystyle{iclr2025_conference}

\newpage
\appendix
\section*{Appendix}
\label{sec:appendix}

\section{\edit{Additional Background on Conformal Predicition for Robotics}}
\label{sec:conformal_robotics_related}
\para{Conformal Prediction for Robotics}
Recently, conformal prediction has become popular in the robotics domain in part due to the distribution-free guarantees it provides for arbitrarily complex learned models present within modern robotics pipelines. 
Specifically, conformal prediction has been used to provide collision-avoidance assurances \citep{chen2021reactive, lindemann2023safe, dixit2023adaptive, muthali2023multi, taufiq2022conformal, dietterich2022conformal, lin2024verification}, calibrate early warning systems \citep{luo2022sample}, and quantify uncertainty in large language model based planners \citep{ren2023robots, lidard2024risk}. 
There are several core challenges with the input-and-label data encountered in robotics: data is non-i.i.d. (e.g., sequential decision-making), data distributions are non-stationary (e.g. changing environment conditions), and labels are intermittently observed (e.g., limited expert feedback in the IL domain). 
By extending online conformal prediction to the intermittent label setting, we take a step towards addressing these challenges.

\section{\edit{Hyperparameter Sensitivity Analysis}}
\label{sec:hyperparam_analysis}
\begin{figure}[ht]
    \centering
    \includegraphics[width=0.9\linewidth]{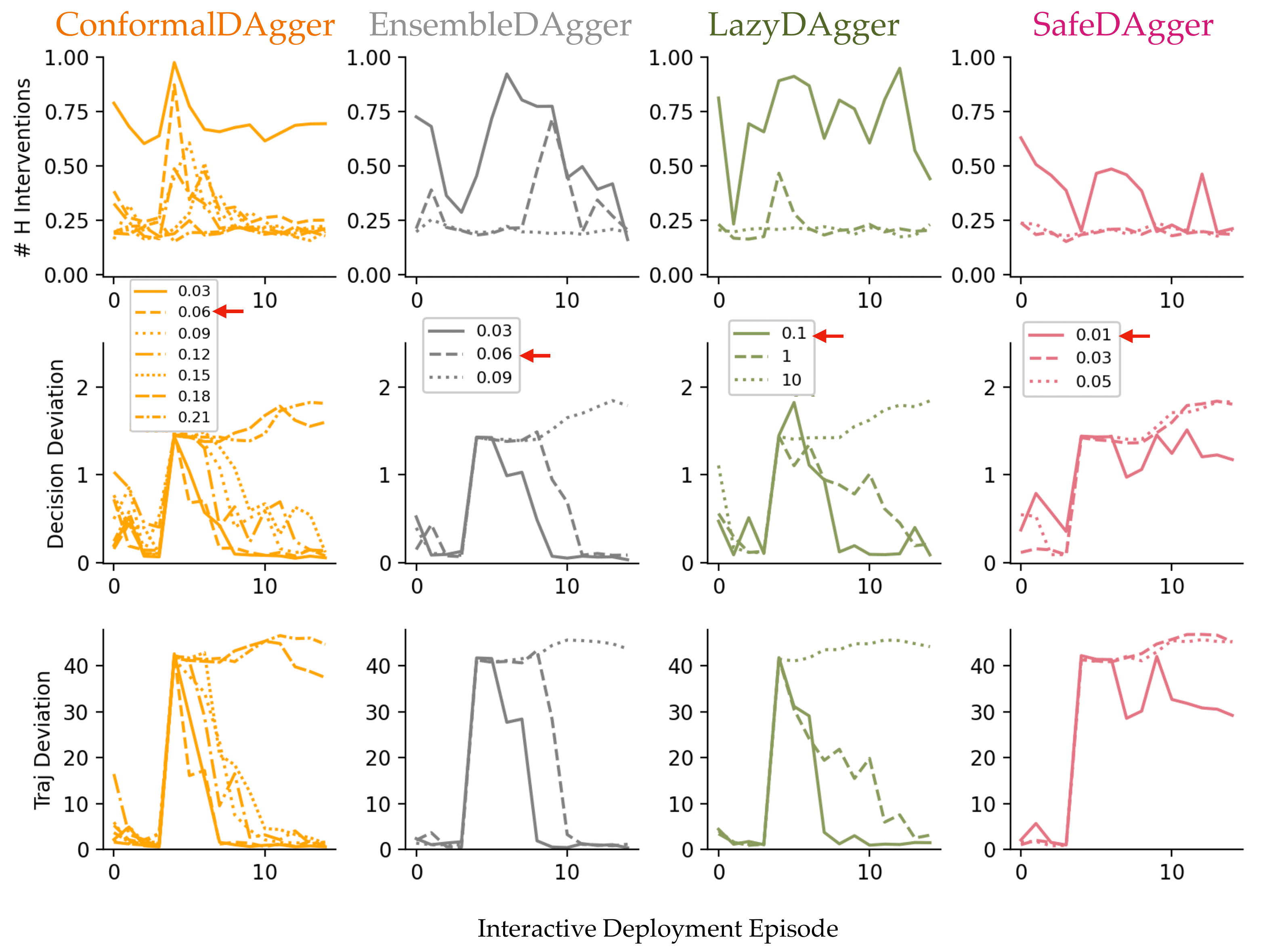}
    \caption{Hyperparameter sensitivity analysis under Expert Shift (Human-gated probability $P(obs^h_t|o_t) = 0.2$).}
    \label{fig:sens}
\end{figure}

We show below a sensitivity analysis in the Expert Shift (Human-gated probability $P(obs^h_t|o_t) = 0.2$) context of the uncertainty threshold for \cdagger and \ensemble, where we vary the uncertainty threshold between 0.03 and 0.09 (Figure \ref{fig:sens}). To demonstrate the results over the deployment episodes, we hold the \ensemble safety threshold constant at 0.03. The results are averaged over 3 seeds. 

When the threshold is low (0.03), the approaches all identify the shift and ask for help accordingly; however, this comes at the cost of a great deal of human feedback requested throughout every deployment episode. Even in the first episode in which there has been no expert shift, a low-threshold at 0.03 requires conformal and ensemble to ask questions 70-80\% of the time. When the threshold is set too high, at 0.09, \ensemble is unable to identify the shift and continuously mispredict. 

We included in our sensitivity analysis \lazy’s context-switching threshold, where we tested values of 0.1, 1, and 10, which are thresholds presented in the \lazy sensitivity analysis in \cite{hoque2021lazydagger}. Similarly, we hold the \lazy safety threshold constant at 0.03. For \safe, we examined safety thresholds of 0.01, 0.03, and 0.05. 

Across all of the algorithms, ConformalDAgger’s expert deviation is the least prone to threshold selection, where all models are able to effectively adapt to the expert’s shifted policy. \edit{Empirically, we find that if $\tau$ is set too high (at 0.18, 0.21), ConformalDAgger may not flag high uncertainty, relying instead on human-gated interventions to update IQT. When we ablate the values of tau between the values [0.03, 0.06, 0.09, 0.12, 0.15], we found ConformalDAgger to be fairly robust in terms of miscoverage (see updated Figure 7 in Appendix B, $\tau$=0.18, 0.21 included as well). 

Theoretically, if $\tau$ is set too high, failing to properly reflects the distribution of encountered misprediction residuals, ConformalDAgger will not flag high uncertainty, relying on human-gated interventions only ($P(obs^h_t|o_t)$) to update IQT and provide retraining data, which may result in poor coverage. However, importantly, one benefit of IQT is that its theoretical guarantees afford us an understanding of what the miscoverage gap will be. Recall that Proposition \ref{prop:iqt_coverage} gives us a bound on the miscoverage gap between the achieved miscoverage rate and desired miscoverage rate $\alpha$. This gap depends on 1) the time horizon $T$ that we are running our algorithm, and 2) the maximum value of $\frac{\gamma_t}{p_t}$. In cases where $\tau$ is set too high such that ($P(obs^r_t|o_t)$ is close to 0), meaning the robot almost never queries for additional help, the probability of getting feedback $p_t$ is approximately the $P(obs^h_t|o_t)$, relying only on the likelihood of human feedback. For short sequences ($T$ small), and infrequent human-gated feedback ($p_t$ small), Proposition \ref{prop:iqt_coverage} informs us that this miscoverage gap might be large. Proposition \ref{prop:iqt_coverage} shows the theoretical basis for that when human-gated feedback is infrequent, and highlights that a proper and effective selection of $\tau$ to raise $P(obs^r_t|o_t)$ and subsequently $p_t$, helps to reduce the miscoverage gap.}

The red arrows indicate the threshold values we ended up using for our 4D reaching experiments. We used the number of human interventions requested as a proxy metric for helping us choose the thresholds. Although the \ensemble threshold of 0.03 more quickly learns the expert policy than 0.06, however, the 0.03 threshold causes the robot to asks many more questions at the initial episode (even without the existence of shift). We chose the threshold of 0.06 for \ensemble and \cdagger because in the first 4 episodes before the shift occurs, the algorithm asks some, but not excessive, questions. \lazy uses the safety classifier prediction, $s=0.03$ to begin human intervention, and only switches back to autonomous mode when the learner’s prediction and human ground truth action received online are below a context-switching threshold, $0.1s$. Because \safe alone under $s=0.03$ does not ask as many questions as \ensemble and \lazy which both involve additional mechanisms for querying alongside the safety classifier, we choose $s=0.01$ for \safe.

\section{Proof of Proposition \ref{prop:iqt_coverage}}
\label{sec:iqt_proofs}

We will start by proving the following lemma.
Recall that $B$ is the upper bound on $q_1 \in [0,B]$, and $s_t \in [0,B]$ by definition.
\begin{lemma}\label{lemma:iqt_bounded}
For all $t$, we have
    $-\alpha N_{t-1} \leq q_t \leq B+(1-\alpha)N_{t-1}$,
where $N_t = \max_{1\leq r \leq t} \frac{\gamma_r}{p_r}$. $B$ is the upper bound on $q_1 \in [0,B]$ and $s_t \in [0,B]$.  
\end{lemma}
\begin{proof}
     $q_1 \in [0,B]$ by assumption, so the lemma is satisfied at $t=1$. Assume $q_t \in [-\alpha N_{t-1}, B+(1-\alpha)N_{t-1}]$. Now, for $q_{t+1}$, we consider Case 1, where $\obst = 0$, which means that $q_{t+1} =  q_t$, so $q_{t+1}$ lies within the range $[-\alpha N_{t-1}, B+(1-\alpha)N_{t-1}]$. 
     Since $N_t \geq N_{t-1}$, $q_{t+1}$ also lies within the larger range $[-\alpha N_{t}, B+(1-\alpha)N_{t}]$, as desired.
     In Case 2 where $\obst = 1$,  $q_{t+1} =  q_t + \frac{\lrt}{p_t}(\errt - \alpha)$. If we represent $\eta_t := \frac{\gamma_t}{p_t}$, we obtain the constant-feedback quantile tracking update: $q_{t+1} =  q_t + \eta_t(\errt - \alpha)$, with variable $\eta_t$ instead of $\lrt$. Then, Lemma 1 in \citet{angelopoulos2024online}, with $N_t = \max_{1\leq r \leq t} \eta_t$ bounds $q_{t+1}$ within the range $[-\alpha N_{t}, B+(1-\alpha)N_{t}]$. 
\end{proof}

Next, we proceed to show how we can leverage the existing quantile tracking results to derive our results with intermittent feedback.

\begin{proof}
First, we take expectations only with respect to $\obst$ conditional on $x_t$, $\errt$, and all other randomness, noting in particular that $\obst$ is independent of everything else given $x_t$ and has a Bernoulli distribution with mean $p_t$.

\edit{We will abbreviate the left-hand side by $\EE[q_{T+1}|D_{T+1}]$, where $D_{T+1} := \{\errt, x_t\}_{t \leq T+1}$, and use $\EE_{\obs_{1:t}} = \EE_{\obs_1 \sim p_1, ..., \obst \sim p_t}$ to denote this conditional expectation
for brevity below:}
\begin{align}
    &\EE_{\obs_{1:t}} [q_{T+1} - q_{r} | D_{T+1}] = \EE_{\obs_{1:t}} [q_{T+1} | D_{T+1}] - \EE_{\obs_{1:r}}[q_{r} | D_{r}]\\
    &= \EE_{\obs_{1:t}} [q_{1} + \sum_{t=1}^T \frac{\gamma_t}{p_t} (\errt - \alpha)\obst | D_{T+1}] - \EE_{\obs_{1:r}}[q_{1} + \sum_{t=1}^r \frac{\gamma_t}{p_t} (\errt - \alpha)\obst | D_{r}]\\
    &= q_{1} + \sum_{t=1}^T \frac{\gamma_t}{p_t} (\errt - \alpha)p_t - \big( q_{1} + \sum_{t=1}^r \frac{\gamma_t}{p_t} (\errt - \alpha)p_t \big)\\
    &= \sum_{t=r}^T \gamma_t (\errt - \alpha).
\end{align}
Given the definition of $\Delta$, we have that $\gamma_t^{-1} = \sum_{r=1}^t \Delta_r$ for all $t \geq 1$.  So
\begin{align}
    \left\vert\frac{1}{T} \sum_{t=1}^T  (\errt - \alpha) \right\vert  &= |\frac{1}{T} \sum_{t=1}^T  \big( \sum_{r=1}^t \Delta_r \big) \gamma_t (\errt - \alpha) |\\
    &= \left\vert \frac{1}{T} \sum_{r=1}^T \Delta_r \big(\sum_{t=r}^T   \gamma_t (\errt - \alpha) \big) \right\vert\\
    &= \left\vert \frac{1}{T} \sum_{r=1}^T \Delta_r \big(\EE_{\obs_{1:t}} [q_{T+1} - q_{r} |D_{T+1}] \big) \right\vert.
\end{align}
 By Lemma~\ref{lemma:iqt_bounded}, this expected difference is bounded by $B+(1-\alpha)N_{T} - (-\alpha N_{T})$:
\begin{equation}
    \EE_{\obs_{1:T+1}} [q_{T+1} - q_{r}|D_{T+1}] \leq B + \max_{1 \leq t \leq T} \frac{\gamma_t}{p_t}.
\end{equation}
\edit{In other words, we can drop the expectation via Lemma \ref{lemma:iqt_bounded} and consider the worst case bound on $q_{T+1} - q_{r}$.} 
Thus, 
\begin{align}
    \left\vert\frac{1}{T} \sum_{t=1}^T  \errt - \alpha \right\vert &\leq   \frac{1}{T} \sum_{r=1}^T |\dl_r| \big( B + \max_{1 \leq t \leq T} \frac{\gamma_t}{p_t} \big) \\
    & = \frac{1}{T} ||\dl_{1:T}||_1 \big(B + \max_{1\leq t \leq T} \frac{\gamma_t}{p_t}\big) .
\end{align}
This completes the proof.
\end{proof}

\subsection{Special Case: When $\gamma_t = p_t$}
\label{ap_subsec:intermit_qtrack_gamma_p}
Next, given the quantile tracking update with an intermittent observation model, we consider what would happen if $\gamma_t$ was set as $p_t$. When $\gamma_t = p_t$, then quantile tracking update becomes $q_{t+1} =  q_t + (\errt - \alpha) \obst$ and $\max_{1\leq t \leq T} \frac{\gamma_t}{p_t}=1$.

Under Proposition \ref{prop:iqt_coverage}, IQT with $\gamma_t = p_t$ gives the following finite time coverage bound:
\begin{equation}
    \left\vert \frac{1}{T} \sum_{t=1}^T  \errt - \alpha \right\vert \leq \frac{B + 1}{T} ||\Delta_{1:T}||_1
\end{equation}
where the sequence $\Delta$ is defined with values $\Delta_1 = \gamma_1^{-1}$ and $\Delta_t = \gamma_t^{-1} - \gamma_{t-1}^{-1}$ for all $t\geq 2$.

\section{Intermmittent Adaptive Conformal Inference}
\label{sec:iaci_proofs}
We show in this section that intermittent observation of ground truth labels can be extended to Adaptive Conformal Inference (ACI) \citep{gibbs2021adaptive}. To facilitate understanding, we briefly summarize ACI and discuss our extension Intermittent Adaptive Conformal Inference (IACI).

\para{Setup: Quantile Regression (QR)} Similar to IQT, consider an \textit{arbitrary} sequence of data points $(x_t, y_t) \in \XX \times \YY$, for $t = 1, 2,...$, that are not necessarily I.I.D. Our goal in ACI is to also produce prediction sets on the output of any base prediction model such that the sets contain the true label with a specified miscoverage rate $\alpha$. Mathematically, at each time $t$, we observe $x_t$ and seek to cover the true label $y_t$ with a set $C_t(x_t)$, which depends on a base prediction model, $\hat{f}: \XX \rightarrow \YY$. We will discuss ACI with a conformal quantile regression \cite{romano2019conformalized} backbone.
The base model takes as input the current $x_t$ and outputs prediction $\yhat$ as well the \textit{estimated upper and lower conditional quantiles}:
\begin{equation}
     \{\estqlo(x_t), \hat{y}_t, \estqhi(x_t)\} \leftarrow \hat{f}(x_t), ~~\forall (x_t, y_t),
\end{equation}
where $\estqlo(x_t)$ is an estimate of the $\alpha_{lo}$-th conditional quantile and $\estqhi(x_t)$ is the $\alpha_{hi}$-th quantile estimate. During training, $\estqlo(x_t)$ is learned with an additional Pinball loss \cite{koenker1978regression, romano2019conformalized}.  

\para{Adaptive Conformal Inference (ACI)} 
At each time $t$, we compute the nonconformity score $s_t$ as:
\begin{equation}
    s(x_t, y_t; f) = \max\{\estqlo(x_t) - y_t, y_t - \estqhi(x_t)\},
    \label{eq:conformity-score}
\end{equation}
which is the coverage error induced by the regressor's quantile estimates. 

Let $S_t$ be the set of conformity scores for all data points through time $t$ in $\dcalib$. 
In general, the magnitude of $s(x_t, y_t; \hat{f})$ is determined by the miscoverage error and its sign is determined by if the true value of $y_t$ lies outside or inside the estimated interval. 
\\\\
Mathematically, the calibrated prediction interval for $Y_{t}$ is:
\begin{align}
    C_t(x_{t}) = \big[&\estqlo(x_{t}) - \hat{Q}_{S_t}(1-\alpha_t), \\~&\estqhi(x_{t}) + \hat{Q}_{S_t}(1-\alpha_t) \big], 
\end{align}
where the $\hat{Q}_{S_t}(1-\alpha_t) := (1-\alpha_t)(1+\frac{1}{|\dcalib|})$-th adaptive empirical quantile of $S_t = S_{t-1} \cup s(x_t, y_t; \hat{f})$. The empirical quantile is defined as the following (where $k$ is the lookback window):
\begin{equation}\hat{Q}_{S_t}(c) := \text{inf} \Biggl\{m: \left( \frac{1}{|D_{t-k:t}|} \sum_{(x_i, y_i) \in D_{t-k:t}} \ind_{\{s(X_i, Y_i) \leq m\}} \right) \geq c \Biggr\}
\label{eq:aci-empirical-quantile}
\end{equation}

Given the non-stationarity of the data distribution, ACI examines the empirical miscoverage frequency of the previous interval, and then decreases or increases a time-dependent $\alpha_t$. Fixing step size parameter $\gamma > 0$, ACI updates
\begin{equation}
    \alpha_{t+1} \coloneqq \alpha_t + \gamma(\alpha - \errt)
\end{equation}

\para{ACI Coverage Guarantee}
The adaptive quantile adjustments made in ACI provide the following coverage guarantee:
\begin{equation}
    \left\vert \frac{1}{T} \sum_{t=1}^T \errt - \alpha   \right\vert \leq \frac{\max\{\alpha_1, 1-\alpha_1\} + \gamma}{T\gamma}
\end{equation}
obtaining the desired $\alpha$ coverage frequency without making an assumptions on the data-generating distribution (Proposition 4.1 in \cite{gibbs2021adaptive}). As $T$ approaches $\infty$, $\lim_{T\rightarrow \infty} \frac{1}{T} \sum_{t=1}^T \errt$ approaches $\alpha$.
This guarantees ACI gives the $1-\alpha$ long-term empirical coverage frequency regardless of the underlying data generation process.

\subsection{Intermittent Adaptive Conformal Prediction}
To achieve Intermittent Adaptive Conformal Inference (IACI), we update $\alpha_t$ at each timestep $t$ with Equation \ref{eq:iaci_update}:

\begin{equation}
    \alpha_{t+1} = \alpha_t + \frac{\gamma }{p_t} (\alpha - \errt) \obst
\label{eq:iaci_update}
\end{equation}

Recall $\errt=\ind_{y_t \notin C_t(x_t)}$ and $\obst$ represents whether $y_t$ was observed at timestep $t$, and $p_t = \PP(\obs_t=1|x_t) \in (0,1]$. 
The calibrated prediction interval for $y_{t}$ becomes:
\begin{equation}
    C_t(x_{t}) = \big[\estqlo(x_{t}) - \hat{Q}_{S^{obs}_t}(1-\alpha_t), \estqhi(X_{t}) + \hat{Q}_{S^{obs}_t}(1-\alpha_t) \big],
\label{eq:interval_iaci}
\end{equation}
where $\hat{Q}_{S^{obs}_t}(1-\alpha_t) := (1-\alpha_t)(1+\frac{1}{|\dcalib|})$-th adaptive empirical quantile of $S_t^{obs}$, the set of nonconformity scores that have been observed, weighted by their probability of observation. Since we will only observe feedback with probability $p_t$ at each timestep, our set of nonconformity values will not be the full set of scores at every timestep. Instead, we have access to some subset of nonconformity scores $S_t^{obs}$, where $p_i$ is the probability of element $s(x_i, y_i;f)$ being in the subset $S_t^{obs}$ for $1\leq i \leq t$. We define $\hat{Q}_{S^{obs}_t}(c)$: 
\begin{equation}
    \hat{Q}_{S^{obs}_t}(c) := \text{inf} \Biggl\{m: \left( \frac{1}{|D_{t-k:t}|} \sum_{(x_i, y_i) \in D_{t-k:t}} \frac{1}{p_i} \cdot \obs_i \cdot \ind_{\{s(x_i, y_i) \leq m\}} \right) \geq c \Biggr\}
\end{equation}

At best, we can say that in expectation, the inclusion criteria for any element in the summation term is equivalent for the summation in $\hat{Q}^{obs}_t(x)$ and the summation in $\hat{Q}_t(x)$: $\EE_{\obs_i \sim p_i}[\frac{1}{p_i} \cdot \obs_i \cdot \ind_{\{s(X_i, Y_i) \leq m\}} ] =  \ind_{\{s(X_i, Y_i) \leq m\}}$. 

\begin{proposition} \textbf{IACI Coverage Guarantee}.
\label{prop:iaci_coverage}
With this intermittent feedback update, with probability one, where $M_{t} = \text{min }\{p_1,..., p_{t}\} \in (0,1]$, we have that for all $T \in \NN$:
\begin{equation}
    \left\vert \frac{1}{T} \sum_{t=1}^T \errt - \alpha   \right\vert \leq \frac{\max\{\alpha_1, 1-\alpha_1\} + \frac{\gamma}{M_{T+1}}}{T\gamma}
\end{equation}
\label{prop:iaci_coverage}
\end{proposition}
At $\lim_{T\rightarrow \infty} \frac{1}{T} \sum_{t=1}^T \errt$ approaches $\alpha$, given $M_{\infty} = \text{min }\{p_1,..., p_{\infty}\} \in (0,1]$.
This guarantees IACI gives the $1-\alpha$ long-term empirical coverage frequency regardless of the underlying data generation process. The proof follows from \citet{gibbs2021adaptive} as is described in the next subsection (Section \ref{ap_sec:proof_iaci}).

\subsection{Proof for Intermittent Adaptive Conformal Inference}
\label{ap_sec:proof_iaci}
\textbf{Assumptions.} We will assume throughout that with probability one, $\alpha_1 \in [0,1]$, $\alpha \in (0,1)$, $p_t\in (0,1] ~\forall 1\leq t \leq \infty$, and $\hat{Q}^{obs}_t(x)$ is non-decreasing with $\hat{Q}_{S^{obs}_t}(c) = -\infty$ for all $c < 0$, and $\hat{Q}_{S^{obs}_t}(c) = \infty$ for all $c > 1$. 

\begin{lemma}
With probability one, we have that $\forall t\in [1,T]$,
\begin{equation}
    \alpha_t \in [-\frac{\gamma}{M_T}, 1+\frac{\gamma}{M_T}]
\end{equation}
where $M_T = \min_{1 \leq r \leq T} p_r$.
\label{lemma:alpha_bounded}
\end{lemma}

\begin{proof}
Observe that 
\begin{align}
    \sup_{1\leq t \leq T} |\alpha_{t+1}- \alpha_t| &= \sup_{1\leq t \leq T} | \frac{\gamma}{p_t} (\alpha - \errt)| \\ 
     & \leq \sup_{1\leq t \leq T} \frac{\gamma}{M_T} |(\alpha - \errt)| = \frac{\gamma}{M_T} \sup_{1\leq t \leq T} |(\alpha - \errt)|\\ 
     & < \frac{\gamma}{M_T} 
\end{align}

The rest of the proof follows from Lemma 4.1 in \citet{gibbs2021adaptive}. We will write it out explicitly here.

\para{Lower Bound}
Assume towards contradiction that with positive probability the set $\{\alpha_t\}_{t\in [1,T]}$ is such that $\inf_{t\in [1,T]} \{\alpha_t\} < -\frac{\gamma}{M_T}$, which means there exists some element in the set $\{\alpha_t\}_{t\in [1,T]}$ such that the element is less than $-\frac{\gamma}{M_T}$. Let $\alpha_{t+1}$ be the first element in set $\{\alpha_t\}_{t\in [1,T]}$ such that $\alpha_{t+1} < -\frac{\gamma}{M_T}$. 

We know by assumption that $\alpha_1 \in [0,1]$, so $\alpha_{t+1}$ cannot be the first value. So $\alpha_t > \alpha_{t+1}$, else the latter would not be the first element such that  $\alpha_{t+1} < -\frac{\gamma}{M_T}$. Since we know that $ \alpha_t - \alpha_{t+1} < \frac{\gamma}{M_T}$ (from the observation above), then 
\begin{align}
    \alpha_t - \alpha_{t+1} &< \frac{\gamma}{M_T} \\
    \Leftrightarrow \quad 
    \alpha_{t}  & <  \frac{\gamma}{M_T} + \alpha_{t+1} \\ 
     \Leftrightarrow \quad 
    \alpha_{t}  & <  \frac{\gamma}{M_T} + \alpha_{t+1} <  \frac{\gamma}{M_T} + (-\frac{\gamma}{M_T}) \\ 
     \Leftrightarrow \quad 
    \alpha_{t}  & <  0. 
\end{align}
However, if $\alpha_t < 0$, then $\hat{Q}_{S^{obs}_t}(1-\alpha_t) = \infty \Rightarrow \errt = 0$. This is because the quantile will be the trivially large infinite quantile, meaning there will definitely be no undercoverage at $t$.
Since $\frac{\gamma}{p_t} \alpha$ is positive by definition,
\begin{equation}
    \Rightarrow \alpha_{t+1} = \alpha_t + \frac{\gamma}{p_t}  (\alpha - \errt) = \alpha_t + \frac{\gamma}{p_t} (\alpha - 0) \geq \alpha_t
\end{equation}
We have reached a contradiction with $\alpha_{t+1}$ being the first infimum value reached.
\\

\para{Upper Bound}
The upper bound argument is symmetric but we will write it out explicitly.
Assume towards contradiction that with positive probability the set $\{\alpha_t\}_{t\in [1,T]}$ is such that $\sup_{t\in [1,T]} \{ \alpha_t\} > 1+\frac{\gamma}{M_T}$, which means there exists some element in the set $\{\alpha_t\}_{t\in [1,T]}$ such that the element is greater than $1+\frac{\gamma}{M_T}$. Let $\alpha_{t+1}$ be the first element in set $\{\alpha_t\}_{t\in [1,T]}$ such that $\alpha_{t+1} > 1+\frac{\gamma}{M_T}$. 

We know by assumption that $\alpha_1 \in [0,1]$, so $\alpha_{t+1}$ cannot be the first value. So there exists some $\alpha_t < \alpha_{t+1}$, where $\alpha_{t+1} > 1+\frac{\gamma}{M_T}$. Since we know that $ \alpha_{t+1} - \alpha_t <\frac{\gamma}{M_T}$ (from the observation above), then 
\begin{align}
    \alpha_{t+1} - \alpha_t &< \frac{\gamma}{M_T} \\
 \Leftrightarrow \quad    -\alpha_{t}  & <  \frac{\gamma}{M_T} - \alpha_{t+1} \\ 
\Leftrightarrow \quad     \alpha_{t}  & >  \alpha_{t+1} - \frac{\gamma}{M_T} >  (1+\frac{\gamma}{M_T}) - \frac{\gamma}{M_T} \\ 
  \Leftrightarrow \quad   \alpha_{t}  & > 1 . 
\end{align}
However, if $\alpha_t > 1$, then $\hat{Q}_{S^{obs}_t}(1-\alpha_t) = -\infty \Rightarrow \errt = 1$. This is because the quantile will be the trivially small negative infinite quantile, meaning there will definitely be miscoverage at $t$.
Then,
\begin{equation}
    \Rightarrow \alpha_{t+1} = \alpha_t + \frac{\gamma}{p_t} (\alpha - \errt) = \alpha_t + \frac{\gamma}{p_t} (\alpha - 1) \leq \alpha_t
\end{equation}
since $\frac{\gamma}{p_t} (\alpha - 1)$ is is negative by definition of $\alpha$. We have reached a contradiction with $\alpha_{t+1}$ being the first supremum value reached.
\end{proof}

Lemma \ref{lemma:alpha_bounded} will enable us to prove Proposition \ref{prop:iaci_coverage}. The proof is derivative of the constant feedback ACI proof in \citep{gibbs2021adaptive}, and the key idea is to bound the expectation of $\alpha_{t+1}$.
\begin{proof}
Examine the expectation of $\alpha_{t+1}$, conditional on $D_{T+1} := \{\errt, p_t\}_{t \leq T+1}$:
\begin{align}
    \EE_{\obs_1 \sim p_1,..., \obst \sim p_{T+1} }[\alpha_{T+1}| D_{T+1}]  &= \alpha_1 + \gamma \sum_{t=1}^{T+1}(\alpha - \err_t) 
\end{align}

We will abbreviate the left hand side by $\EE[\alpha_{T+1}|D_{T+1}]$.
Because the expected value cannot exceed the range of the value of $\alpha_{T+1}$, we infer that $\EE[\alpha_{T+1}|D_{T+1}] \in [-\gmtp, 1+ \gmtp]$. 

First, we observe from Lemma \ref{lemma:alpha_bounded} that 
\begin{align}
    \EE[\alpha_{T+1}|D_{T+1}] &=  \alpha_1 + \sum_{t=1}^T \gamma(\alpha - \errt) \in [-\gmtp, 1+ \gmtp] \\
    \implies \quad - \sum_{t=1}^T \gamma (\alpha - \errt)  &=  \alpha_1  - \EE[\alpha_{T+1}|D_{T+1}] \\
    \implies \quad \left\vert\frac{1}{T}\sum_{t=1}^T \errt - \alpha \right\vert &= \frac{ \left\vert \alpha_1  - \EE[\alpha_{T+1}|D_{T+1}]\right\vert}{T\gamma} .
\end{align}
To bound the right hand side above, consider in turn the two cases of $\alpha_1  - \EE[\alpha_{T+1}|D_{T+1}] \geq 0$ and $\alpha_1  - \EE[\alpha_{T+1}|D_{T+1}] < 0$.

Starting with Case 1, 
\begin{equation}
    \alpha_1  - \EE[\alpha_{T+1}|D_{T+1}] \geq 0 \Rightarrow \alpha_1  \geq \EE[\alpha_{T+1}|D_{T+1}] \Rightarrow \alpha_{T+1} \in [-\gmtp, \alpha_1].
\end{equation}
    
This case corresponds to the following equivalence: $\left\vert \alpha_1  - \EE[\alpha_{T+1}|D_{T+1}]\right\vert = \alpha_1  - \EE[\alpha_{T+1}|D_{T+1}] $.
    
Negating $\alpha_{T+1} \geq -\gmtp$, we get $-\alpha_{T+1} \leq \gmtp$. Thus,
\begin{align}
    \left\vert \alpha_1  - \EE[\alpha_{T+1}|D_{T+1}]\right\vert &= \alpha_1  - \EE[\alpha_{T+1}|D_{T+1}] \leq \alpha_1 + \gmtp \\
   \implies\quad  \frac{\left\vert \alpha_1  - \EE[\alpha_{T+1}|D_{T+1}]\right\vert}{T\gamma} &\leq \frac{\alpha_1 + \gmtp}{T\gamma}.
\end{align}

In Case 2, 
\begin{equation}
    \alpha_1  - \EE[\alpha_{T+1}|D_{T+1}] < 0 \Rightarrow \alpha_1  < \EE[\alpha_{T+1}|D_{T+1}] \Rightarrow \EE[\alpha_{T+1}|D_{T+1}] \in (\alpha_1, 1+\gmtp].
\end{equation}
This case corresponds to the following equivalence: $\left\vert \alpha_1  - \EE[\alpha_{T+1}|D_{T+1}]\right\vert = \EE[\alpha_{T+1}|D_{T+1}] - \alpha_1 = -1(\alpha_1  - \EE[\alpha_{T+1}|D_{T+1}) $.
Plugging in $\EE[\alpha_{T+1}|D_{T+1}] < 1+\gmtp$, 
\begin{align}
    \left\vert \alpha_1  - \EE[\alpha_{T+1}|D_{T+1}]\right\vert &=  \EE[\alpha_{T+1}|D_{T+1}] - \alpha_1 \leq 1 + \gmtp  - \alpha_1 \\ 
    \frac{\left\vert \alpha_1  - \EE[\alpha_{T+1}|D_{T+1}]\right\vert}{T\gamma} &\leq \frac{(1-\alpha_1) + \gmtp}{T\gamma}.
\end{align}
Lastly, we merge the two cases by taking the maximum over $\{\alpha_1, 1-\alpha_1\}$ to come up with an upper bound that covers both cases.
\begin{equation}
    \left\vert \frac{1}{T} \sum_{t=1}^T \errt - \alpha   \right\vert  = \frac{\left\vert \alpha_1  - \EE[\alpha_{T+1}|D_{T+1}]\right\vert}{T\gamma} \leq \frac{\max\{\alpha_1, 1-\alpha_1\} + \gmtp}{T\gamma}.
\end{equation}
Taking the limit as $T \rightarrow \infty$, if $M_{\infty}$ is bounded, we get $\lim_{T\rightarrow \infty} \left\vert \frac{1}{T} \sum_{t=1}^T \errt - \alpha   \right\vert  = 0 $,
as claimed.
\end{proof}

\section{Extended Experiments: Intermittent Quantile Tracking}
\label{sec:extended_timeseries_iqt}
We experiment with these four predictors because of their use in prior conformal literature \citep{angelopoulos2024conformal} and in order to ensure that under different base model conditions, \iqt maintains coverage close to the desired level. (1) Autoregressive (\textbf{AR}) model with 3 lags, (2) \textbf{Theta} model with $\theta=2$ \citep{assimakopoulos2000theta}, (3) \textbf{Prophet} model \cite{taylor2018forecasting}, and (4) \textbf{Transformer} model \citep{vaswani2017attention}. 
Consistent with prior works, for all base models except for transformer, we retrain the base model after each timestep; for the transformer, we retrain every 100 timesteps. We set lookback window $k=100$ timesteps for the Google and Amazon stock price data, and $k=300$ for the Elec2 dataset.

\para{Amazon stock price data under partial ($p_t=0.5$) feedback}
Figure \ref{fig:time_series_iqt_p05} shows the prediction interval sizes for 1 seed, and coverage averaged over 5 seeds for Amazon stock price data under partial ($p_t=0.5$) feedback. We see that the interval size for \iqtpd is larger than \iqtpi for high learning rate $\constlr=1$, but the size of intervals is comparable for smaller learning rates. Table \ref{tab:amzn_p05} shows the performance metrics averaged over 5 seeds. 
\begin{figure}[thb]
    \centering
    \includegraphics[width=0.99\linewidth]{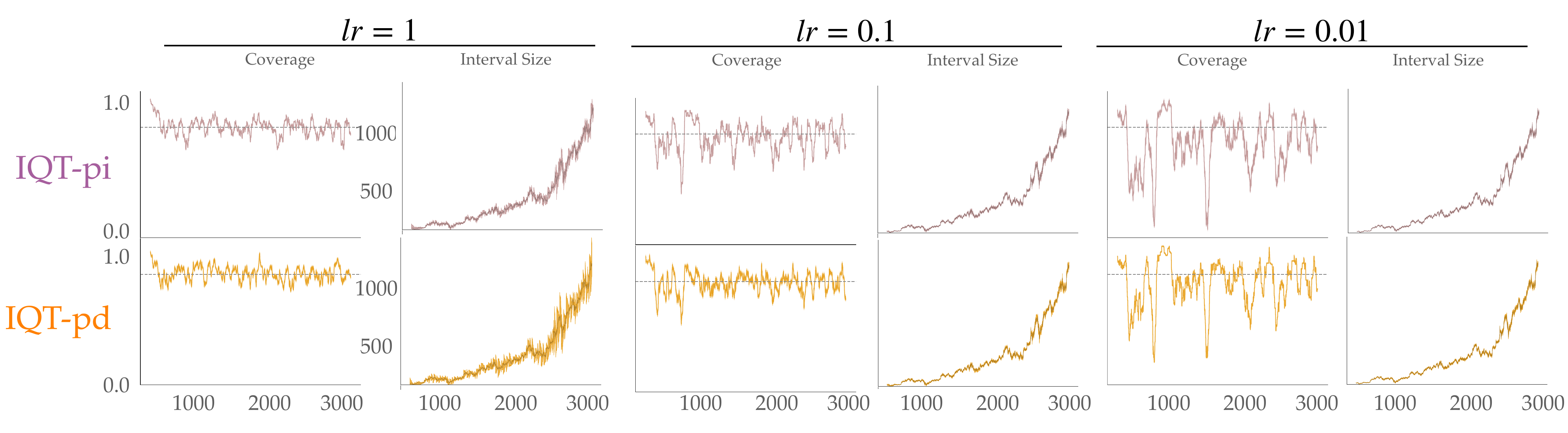}
    \caption{Amazon stock price data under partial ($p_t=0.5$) feedback. We show the prediction interval sizes for 1 seed, and coverage averaged over 5 seeds for Amazon stock price data under partial observations.}
    \label{fig:time_series_iqt_p05}
\end{figure}

\begin{table}[tbp!]
\centering
\small
{
\caption{\textbf{IQT on Amazon Stocks (\textit{infrequent observations}): Trends Across Models.} We test four base models, set $p_t = 0.1, \forall t$ to simulate seeing the true price only 10\% of the time, and report the mean across 5 seeds.}
\label{tab:amzn_p01}
\resizebox{\linewidth}{!}{%
\begin{tabular}{c c |c c | c  c | c  c |c c } 
\toprule 
 \multicolumn{1}{c}{\textbf{Metric}} & \multicolumn{1}{c}{\textbf{lr}} &  \multicolumn{2}{c}{\textbf{AR Base Model}} & \multicolumn{2}{c}{\textbf{Prophet Base Model}} & \multicolumn{2}{c}{\textbf{Theta Base Model}} & \multicolumn{2}{c}{\textbf{Transformer Base Model}}\\
{} & {} & \iqtpi & \iqtpd & \iqtpi &  \iqtpd & \iqtpi & \iqtpd & \iqtpi & \iqtpd\\
 [0.5ex] 
 \hline
Coverage &  1.0 &  0.903&	0.937&	0.906&	0.935	&0.899&	0.921	&0.902	&0.923\\
{}  &  0.1 &  0.876&	0.903	&0.876	&0.906	&0.815	&0.899	&0.808&	0.902\\
{}  &  0.01 &  0.690&	0.876	&0.690&	0.876&	0.430	&0.815&	0.360&	0.808\\
 \hline
Longest err seq &  1.0 &  13.2&	13.2	&11.6	&30.6&	27	&41.2&	28.8	&30\\
{}  &  0.1 &  6.8&	20.8	&10&	11.6&	51.8	&27	&89	&28.8\\
{}  &  0.01 &  12.6	&6.8	&13.8	&10&	145.2	&51.8	&212.8&	89\\
 \hline
Avg set size &  1.0 &  47.452&	425.435	&46.056	&414.977&	105.040	&711.829&	131.025&	982.469\\
{}  &  0.1 &  16.872&	47.452&	16.495	&46.056	&82.463&	105.040&	118.040	&131.025\\
{}  &  0.01 &  9.428	&16.872	&9.470	&16.495&	35.458&	82.463&	46.929&	118.040\\
 \bottomrule
\end{tabular}
}
}
\end{table}

\begin{table}[tbp!]
\centering
\small
{
\caption{\textbf{IQT on Amazon Stocks (\textit{partial observations}): Trends Across Models.} We test four base models, set $p_t = 0.5, \forall t$ to simulate seeing the true price only 50\% of the time, and report the mean across 5 seeds.}
\label{tab:amzn_p05}
\resizebox{\linewidth}{!}{
\begin{tabular}{c c c c  c  c  c  c c c } 
\toprule 
 \multicolumn{1}{c}{\textbf{Metric}} & \multicolumn{1}{c}{\textbf{lr}} &  \multicolumn{2}{c}{\textbf{AR}} & \multicolumn{2}{c}{\textbf{Prophet}} & \multicolumn{2}{c}{\textbf{Theta}} & \multicolumn{2}{c}{\textbf{Transformer}}\\
{} & {} & \iqtpi & \iqtpd & \iqtpi &  \iqtpd & \iqtpi & \iqtpd & \iqtpi & \iqtpd\\
 [0.5ex] 
 \hline
Marginal coverage &  1.0 &  0.902	&0.905	&0.901	&0.904	&0.906&	0.904&	0.900&	0.904\\
{}  &  0.1 &  0.891&	0.896	&0.883&	0.893&	0.890&	0.895&	0.879	&0.885\\
{}  &  0.01 &  0.846&0.869	&0.737	&0.812	&0.843&	0.870	&0.705	&0.796\\
 \hline
Longest err seq &  1.0 & 5.6	&5.8	&7.4	&8	&5.6&	6.6	&7.8	&8\\
{}  &  0.1 &  4.6&	4	&12.8&8.2	&7.6&	5	&23.4&	11.4\\
{}  &  0.01 &  7.2	&7	&89.4	&50.6	&10&	10	&191.2	&97.6\\
 \hline
Avg set size &  1.0 &  43.816&	78.447&	76.013	&139.875	&44.449&	77.741	&100.127&	183.529\\
{}  &  0.1 &  17.535	&19.918	&66.940	&55.233	&17.581&	19.854&	92.513&72.867\\
{}  &  0.01 &  13.684	&14.952	&70.972&	81.854	&13.678	&15.085	&99.109	&118.457\\
 \bottomrule
\end{tabular}
}
}
\end{table}

\para{Amazon stock price data under partial ($p_t=0.9$) feedback}
Figure \ref{fig:time_series_iqt_p09} shows the prediction interval sizes for 1 seed, and coverage averaged over 5 seeds for Amazon stock price data under frequent ($p_t=0.9$) feedback. We see that the interval size for \iqtpd and \iqtpi are very similar across learning rates. Table \ref{tab:amzn_p09} shows the performance metrics averaged over 5 seeds. 
\begin{figure}[thb]
    \centering
    \includegraphics[width=0.99\linewidth]{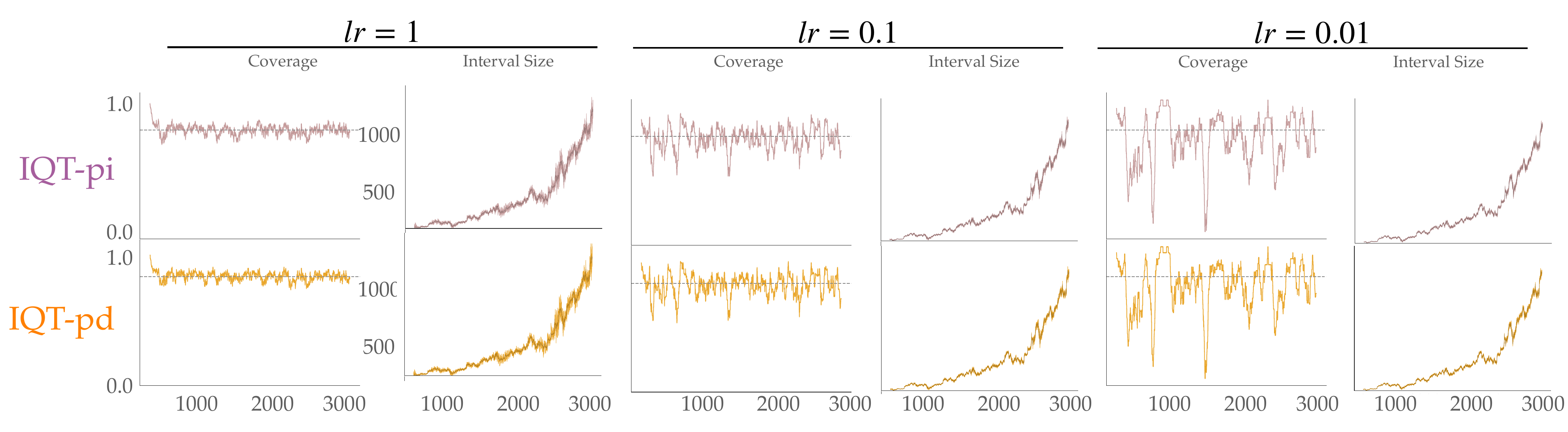}
    \caption{Amazon stock price data under frequent ($p_t=0.9$) feedback. We show the prediction interval sizes for 1 seed, and coverage averaged over 5 seeds for Amazon stock price data under partial observations.}
    \label{fig:time_series_iqt_p09}
\end{figure}

\begin{table}[tbp!]
\centering
\small
{
\caption{\textbf{IQT on Amazon Stocks (\textit{frequent observations}): Trends Across Models.} We test four base models, set $p_t = 0.9, \forall t$ to simulate seeing the true price 90\% of the time, and report the mean across 5 seeds.}
\label{tab:amzn_p09}
\resizebox{\linewidth}{!}{
\begin{tabular}{c c c c  c  c  c  c c c } 
\toprule 
 \multicolumn{1}{c}{\textbf{Metric}} & \multicolumn{1}{c}{\textbf{lr}} &  \multicolumn{2}{c}{\textbf{AR}} & \multicolumn{2}{c}{\textbf{Prophet}} & \multicolumn{2}{c}{\textbf{Theta}} & \multicolumn{2}{c}{\textbf{Transformer}}\\
{} & {} & \iqtpi & \iqtpd & \iqtpi &  \iqtpd & \iqtpi & \iqtpd & \iqtpi & \iqtpd\\
 [0.5ex] 
 \hline
Marginal coverage &  1.0 &  0.902&	0.903	&0.901	&0.903	&0.902	&0.902	&0.903&	0.902\\
{}  &  0.1 &  0.894	&0.894&	0.886&	0.886	&0.8932	&0.894	&0.885	&0.887\\
{}  &  0.01 &  0.870	&0.873	&0.797	&0.805	&0.871&	0.873&	0.786	&0.798\\
 \hline
Longest err seq &  1.0 & 2.8&	2.6&	3	&3.6&	3.2	&3.6	&3.2&	3\\
{}  &  0.1 &  3.6&	3.8	&7	&6.6	&5.4	&5&	10.6&	9\\
{}  &  0.01 &  7&	7	&57.8&	44.2&	10	&10	&97.8&	98.4\\
 \hline
Avg set size &  1.0 &  43.221	&46.624	&71.646	&78.889&42.979&	46.270	&93.068	&103.777\\
{}  &  0.1 &  17.663	&17.982	&54.235	&52.265&	17.687&	18.121&	69.700	&66.835\\
{}  &  0.01 &  14.627	&14.832&	79.736	&81.078	&14.732	&14.954	&116.272	&117.762\\
 \bottomrule
\end{tabular}
}
}
\end{table}

\subsection{Google Stock Price Dataset Results}
For the Google stock price dataset, Tables \ref{tab:google_p01}, \ref{tab:google_p05}, and \ref{tab:google_p09} show the results for infrequent, partial, and frequent feedback. Under infrequent observations, \iqtpd boosts coverage over \iqtpi, at the cost of larger prediction intervals on average. As the observation probability increases, the discrepancy in prediction interval size between \iqtpd and \iqtpi narrows, and the coverage is very similar. At a high level, the trends in performance for \iqtpd and \iqtpi across the three observation levels are very similar for the Google dataset and the Amazon stock price dataset.

\begin{table}[tbp!]
\centering
\small
{
\caption{\textbf{IQT on Google Stocks (\textit{infrequent observations}): Trends Across Models.} We test four base models, set $p_t = 0.1, \forall t$ to simulate seeing the true price only 10\% of the time, and report the mean across 5 seeds.}
\label{tab:google_p01}
\resizebox{\linewidth}{!}{
\begin{tabular}{c c c c  c  c  c  c c c } 
\toprule 
 \multicolumn{1}{c}{\textbf{Metric}} & \multicolumn{1}{c}{\textbf{lr}} &  \multicolumn{2}{c}{\textbf{AR}} & \multicolumn{2}{c}{\textbf{Prophet}} & \multicolumn{2}{c}{\textbf{Theta}} & \multicolumn{2}{c}{\textbf{Transformer}}\\
{} & {} & \iqtpi & \iqtpd & \iqtpi &  \iqtpd & \iqtpi & \iqtpd & \iqtpi & \iqtpd\\
 [0.5ex] 
 \hline
Marginal coverage &  1.0 & 0.925&	0.949	&0.928&	0.929&	0.929	&0.929	&0.907	&0.927\\
{}  &  0.1 &  0.910&	0.925	&0.897	&0.928	&0.906	&0.929	&0.793&	0.907\\
{}  &  0.01 & 0.812&	0.910&	0.639&	0.897&	0.692&	0.906&	0.234&	0.793\\
 \hline
Longest err seq &  1.0 &  11.4	&27.8&	21.4	&42.4	&17.2&	29.2	&23.2	&29\\
{}  &  0.1 &  6 &	11.4	&35.6	&21.4	&53.4	&17.2	&109.2	&23.2\\
{}  &  0.01 & 9.2	&6	&69.6	&35.6&	103	&53.4&	873.4&	109.2\\
 \hline
Avg set size &  1.0 & 87.834	&911.476	&134.392	&1017.708&	93.068	&815.183	&218.270&	1445.050\\
{}  &  0.1 & 24.305&	87.834	&104.224	&134.392	&49.364	&93.068	&269.950&	218.270\\
{}  &  0.01 &  15.070	&24.305	&54.653&	104.224	&28.910&	49.364	&109.453	&269.950\\
 \bottomrule
\end{tabular}
}
}
\end{table}

\begin{table}[tbp!]
\centering
\small
{
\caption{\textbf{IQT on Google Stocks (\textit{partial observations}): Trends Across Models.} We test four base models, set $p_t = 0.5, \forall t$ to simulate seeing the true price only 50\% of the time, and report the mean across 5 seeds.}
\label{tab:google_p05}
\resizebox{\linewidth}{!}{
\begin{tabular}{c c c c  c  c  c  c c c } 
\toprule 
 \multicolumn{1}{c}{\textbf{Metric}} & \multicolumn{1}{c}{\textbf{lr}} &  \multicolumn{2}{c}{\textbf{AR}} & \multicolumn{2}{c}{\textbf{Prophet}} & \multicolumn{2}{c}{\textbf{Theta}} & \multicolumn{2}{c}{\textbf{Transformer}}\\
{} & {} & \iqtpi & \iqtpd & \iqtpi &  \iqtpd & \iqtpi & \iqtpd & \iqtpi & \iqtpd\\
 [0.5ex] 
 \hline
Marginal coverage &  1.0 &  0.909&	0.909&	0.902&	0.908	&0.906	&0.903&	0.902&	0.899\\
{}  &  0.1 &  0.905	&0.905&	0.896	&0.892	&0.899&	0.901&	0.901	&0.894\\
{}  &  0.01 & 0.891	&0.901	&0.858	&0.887	&0.864&	0.890	&0.661&	0.809\\
 \hline
Longest err seq &  1.0 &  5.2&	6.4	&8	&7.6&	5.4	&7.8	&8.2&	8.4\\
{}  &  0.1 &  4	&4	&13.8	&11	&13.8&	7.4&	16&	13.2\\
{}  &  0.01 &  6	&6	&49.4&	32.2&	80.8&	55.6&	239.8&	104.2\\
 \hline
Avg set size &  1.0 &  59.236	&101.424	&91.522&	168.522&	64.736&	116.381	&131.415	&241.408\\
{}  &  0.1 &  22.021	&25.670&	71.972	&63.493	&36.04	&34.893	&156.157	&106.766\\
{}  &  0.01 &  18.820	&19.830	&91.394	&96.245	&42.575	&43.364	&246.053&	276.761\\
 \bottomrule
\end{tabular}
}
}
\end{table}

\begin{table}[tbp!]
\centering
\small
{
\caption{\textbf{IQT on Google Stocks (\textit{frequent observations}): Trends Across Models.} We test four base models, set $p_t = 0.9, \forall t$ to simulate seeing the true price 90\% of the time, and report the mean across 5 seeds.}
\label{tab:google_p09}
\resizebox{\linewidth}{!}{
\begin{tabular}{c c c c  c  c  c  c c c } 
\toprule 
 \multicolumn{1}{c}{\textbf{Metric}} & \multicolumn{1}{c}{\textbf{lr}} &  \multicolumn{2}{c}{\textbf{AR}} & \multicolumn{2}{c}{\textbf{Prophet}} & \multicolumn{2}{c}{\textbf{Theta}} & \multicolumn{2}{c}{\textbf{Transformer}}\\
{} & {} & \iqtpi & \iqtpd & \iqtpi &  \iqtpd & \iqtpi & \iqtpd & \iqtpi & \iqtpd\\
 [0.5ex] 
 \hline
Marginal coverage &  1.0 &  0.903&	0.902&	0.902&	0.903&	0.902&	0.902&	0.902	&0.901\\
{}  &  0.1 & 0.897&	0.898&	0.889&	0.889	&0.898&	0.898&	0.892&	0.892\\
{}  &  0.01 &  0.895&	0.897&	0.880&	0.881&	0.889	&0.891	&0.791	&0.808\\
 \hline
Longest err seq &  1.0 &  2.8	&3	&3.2&	2.8&	3.6&	2.8	&3	&3.6\\
{}  &  0.1 &  4	&4&	8&	7.8&	6.8&	6.6	&10	&9.4\\
{}  &  0.01 &  6&	6&	32.4&	31.4&	58	&57.6&	108.4	&105\\
 \hline
Avg set size &  1.0 &  53.161	&57.376	&85.427	&93.799	&61.876	&68.177	&121.994	&134.782\\
{}  &  0.1 &  21.068&	21.564	&59.732	&57.680	&31.262	&31.240	&106.240&	99.712\\
{}  &  0.01 &  19.096	&19.396	&94.682	&94.370&	43.206&	42.881	&278.926&	279.124\\
 \bottomrule
\end{tabular}
}
}
\end{table}

\subsection{Elec2 Dataset Results}
For the Elec2 dataset, Tables \ref{tab:elec_p01}, \ref{tab:elec_p05}, and \ref{tab:elec_p09} show the results for infrequent, partial, and frequent feedback. Similar to the other two datasets, under infrequent observations, \iqtpd boosts coverage over \iqtpi, at the cost of larger prediction intervals on average. As the observation probability increases, the discrepancy in prediction interval size between \iqtpd and \iqtpi narrows, and the coverage is very similar. The trends in performance for \iqtpd and \iqtpi across the three observation levels are very similar for the Elec2 dataset and the Amazon stock price dataset.

\begin{table}[tbp!]
\centering
\small
{
\caption{\textbf{IQT on Elec2 Dataset (\textit{infrequent observations}): Trends Across Models.} We test four base models, set $p_t = 0.1, \forall t$ to simulate seeing the true price only 10\% of the time, and report the mean across 5 seeds.}
\label{tab:elec_p01}
\resizebox{\linewidth}{!}{
\begin{tabular}{c c c c  c  c  c  c c c } 
\toprule 
 \multicolumn{1}{c}{\textbf{Metric}} & \multicolumn{1}{c}{\textbf{lr}} &  \multicolumn{2}{c}{\textbf{AR}} & \multicolumn{2}{c}{\textbf{Prophet}} & \multicolumn{2}{c}{\textbf{Theta}} & \multicolumn{2}{c}{\textbf{Transformer}}\\
{} & {} & \iqtpi & \iqtpd & \iqtpi &  \iqtpd & \iqtpi & \iqtpd & \iqtpi & \iqtpd\\
 [0.5ex] 
 \hline
Marginal coverage &  1.0 &  0.918&	0.908&	0.900	&0.934	&0.909	&0.949	&0.903&	0.904\\
{}  &  0.1 &  0.897&	0.918	&0.901	&0.900	&0.911&	0.909	&0.908&	0.903\\
{}  &  0.01 &  0.809&0.897&	0.665&	0.901&	0.857	&0.911	&0.663&	0.908\\
 \hline
Longest err seq &  1.0 &  10.4	&24.8&	19.6&	28.6&	15.2	&24.4&	20.2	&25.4\\
{}  &  0.1 &  4	&10.4&	19.2	&19.6	&9.2	&15.2	&10.8	&20.2\\
{}  &  0.01 &  5.2	&4&	29.8&	19.2&	13	&9.2&	29.4&	10.8\\
 \hline
Avg set size &  1.0 &  0.120	&1.867	&0.846&	6.374	&0.252&	2.489	&0.905	&6.565\\
{}  &  0.1 & 0.091&	0.120	&0.528	&0.846&	0.085	&0.252&	0.556	&0.905\\
{}  &  0.01 &  0.061	&0.091	&0.307&	0.528&	0.059&	0.085&0.349	&0.556\\
 \bottomrule
\end{tabular}
}
}
\end{table}

\begin{table}[tbp!]
\centering
\small
{
\caption{\textbf{IQT on Elec2 Dataset (\textit{partial observations}): Trends Across Models.} We test four base models, set $p_t = 0.5, \forall t$ to simulate seeing the true price only 50\% of the time, and report the mean across 5 seeds.}
\label{tab:elec_p05}
\resizebox{\linewidth}{!}{
\begin{tabular}{c c c c  c  c  c  c c c } 
\toprule 
 \multicolumn{1}{c}{\textbf{Metric}} & \multicolumn{1}{c}{\textbf{lr}} &  \multicolumn{2}{c}{\textbf{AR}} & \multicolumn{2}{c}{\textbf{Prophet}} & \multicolumn{2}{c}{\textbf{Theta}} & \multicolumn{2}{c}{\textbf{Transformer}}\\
{} & {} & \iqtpi & \iqtpd & \iqtpi &  \iqtpd & \iqtpi & \iqtpd & \iqtpi & \iqtpd\\
 [0.5ex] 
 \hline
Marginal coverage &  1.0 &  0.909	&0.905	&0.898	&0.908&	0.907	&0.909&	0.902	&0.899\\
{}  &  0.1 &  0.895	&0.896	&0.901	&0.896	&0.901	&0.899&	0.903	&0.897\\
{}  &  0.01 &  0.892	&0.894&	0.865	&0.887&	0.899&	0.901	&0.886	&0.903\\
 \hline
Longest err seq &  1.0 &  4.6&	5.2	&8.6	&7.4&	5.2	&6.4&	6.8&	7.6\\
{}  &  0.1 &  3.8	&5&	7.2&	7.4	&5.2	&4.8&	7.8	&7.6\\
{}  &  0.01 &  3.2	&2.8	&27&	14.2&	4.2&	4.6	&11	&9.8\\
 \hline
Avg set size &  1.0 &  0.187&	0.325&	0.721&	1.304	&0.234	&0.443	&0.738&	1.287\\
{}  &  0.1 &  0.093&	0.101	&0.475&	0.4721&	0.074&	0.087	&0.487&	0.478\\
{}  &  0.01 &  0.087	&0.088	&0.481	&0.500	&0.068	&0.070	&0.497	&0.517\\
 \bottomrule
\end{tabular}
}
}
\end{table}

\begin{table}[tbp!]
\centering
\small
{
\caption{\textbf{IQT on Elec2 Dataset (\textit{frequent observations}): Trends Across Models.} We test four base models, set $p_t = 0.9, \forall t$ to simulate seeing the true price 90\% of the time, and report the mean across 5 seeds.}
\label{tab:elec_p09}
\resizebox{\linewidth}{!}{
\begin{tabular}{c c c c  c  c  c  c c c } 
\toprule 
 \multicolumn{1}{c}{\textbf{Metric}} & \multicolumn{1}{c}{\textbf{lr}} &  \multicolumn{2}{c}{\textbf{AR}} & \multicolumn{2}{c}{\textbf{Prophet}} & \multicolumn{2}{c}{\textbf{Theta}} & \multicolumn{2}{c}{\textbf{Transformer}}\\
{} & {} & \iqtpi & \iqtpd & \iqtpi &  \iqtpd & \iqtpi & \iqtpd & \iqtpi & \iqtpd\\
 [0.5ex] 
 \hline
Marginal coverage &  1.0 &  0.902	&0.901&	0.901&	0.901&	0.902	&0.902	&0.902	&0.902\\
{}  &  0.1 &  0.900	&0.902&	0.897&	0.898&	0.897&	0.899&	0.897	&0.898\\
{}  &  0.01 & 0.903	&0.903	&0.881	&0.883	&0.893	&0.894	&0.899	&0.902\\
 \hline
Longest err seq &  1.0 &  2.6	&3	&3.8&	3.4	&2.6	&2.8&	3.4	&3.2\\
{}  &  0.1 & 2.6	&2.8	&4.6&	4.8&	4.2&	3.8	&5.8	&5.2\\
{}  &  0.01 & 2.4	&2.8	&15.6	&15	&4.6&	4.8	&10	&10\\
 \hline
Avg set size &  1.0 &  0.179&0.192&	0.677&	0.734	&0.219	&0.243	&0.670	&0.724\\
{}  &  0.1 & 0.096	&0.0977	&0.446	&0.437	&0.070&	0.071	&0.444	&0.438\\
{}  &  0.01 & 0.091&0.092	&0.492	&0.492	&0.066	&0.067	&0.513	&0.518\\
 \bottomrule
\end{tabular}
}
}
\end{table}

\section{Further Examination: Simulated Experiments}
\label{sec:extended_simulated_results}
\para{On Expert Realizability}
An outstanding challenge in the theory of DAgger is guaranteeing learning a high-quality policy from the expert when the expert is unrealizable. We assume, much like \cite{spencer2021feedback}, that our expert is realizable. Whether or not this is true in practice with human experts is an open area of research. Nevertheless, similar to HG-DAgger\cite{kelly2019hg} and EnsembleDAgger\cite{menda2019ensembledagger}, our approach offers a practical approach to contend with these real-world challenges.

\para{Implementation Details}
Both \ensemble and \cdagger are implemented as 7-layer multilayer perceptions (with hidden sizes [64,128,472,512,256,64,42]). We use ReLU activations between each layer. We train using a learning rate of 0.001 and a batch size of 32. We train the initial policy for 200 iterations, and fine-tune between each interactive deployment episode for 100 iterations. The \ensemble safety classifier is implemented as a 4-layer perception ([64,128,64,42]) with ReLU activation between each layer and we apply a sigmoid function on the output to classify the output.

\para{A closer look at expert shift under $p_t = 0.2$ intermittent feedback }
When the expert shift occurs at episode 5 (Figure \ref{fig:shift_at_ep5}), the human provides feedback occasionally to the learners which deviates from the predicted actions. \cdagger increases its calibrated uncertainty based on these human inputs, and once uncertainty exceeds the threshold, the probability for asking for help converges to 1, causing the robot to ask for more human feedback such that it can learn the new goal.
\begin{figure}[thb]
    \centering
    \includegraphics[width=0.99\linewidth]{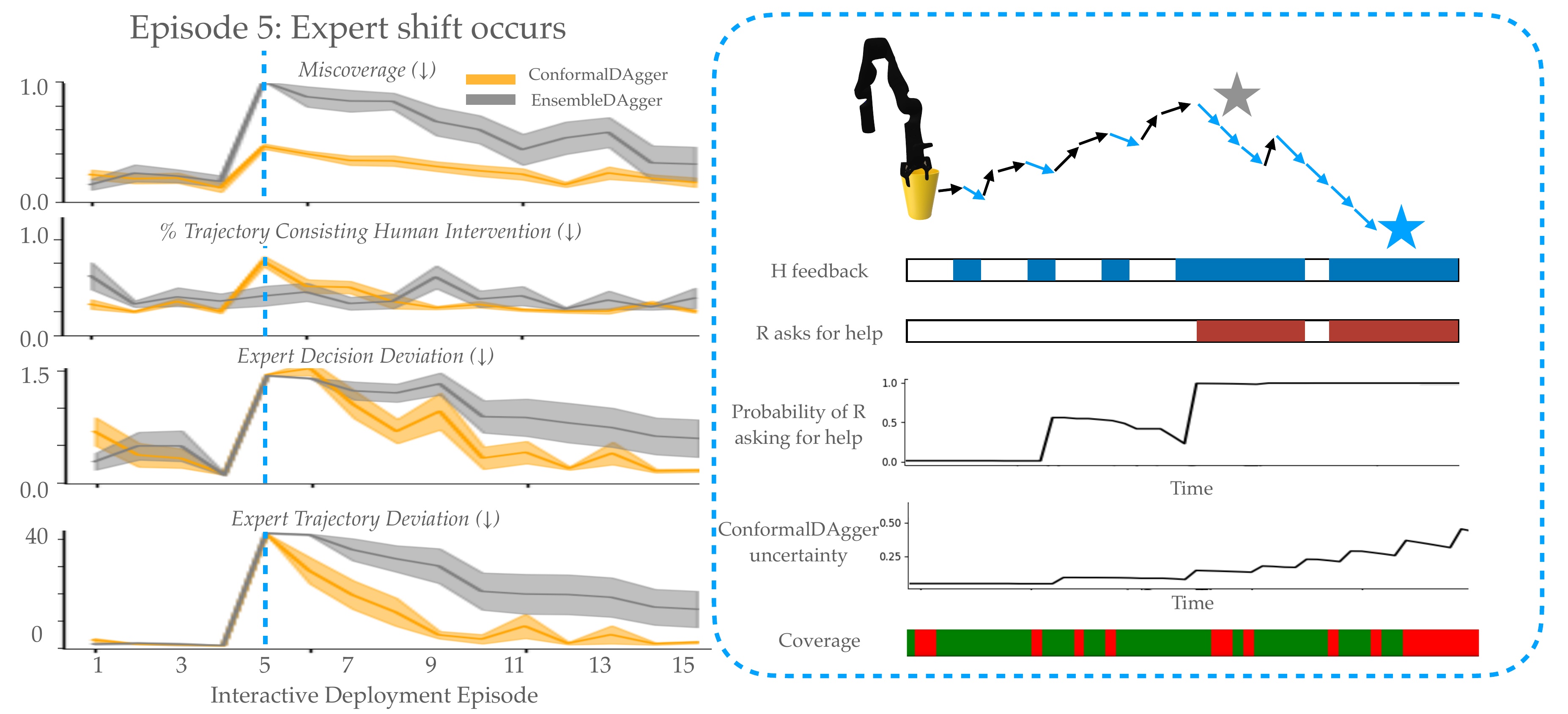}
    \caption{At episode 5, the \cdagger learner uncertainty enables the robot to ask for more human feedback to gather information about the shifted goal.}
    \label{fig:shift_at_ep5}
\end{figure}

\para{Simulated results under partial ($p_t = 0.5$) intermittent feedback }
Under $p_t = 0.5$ intermittent feedback (Figure \ref{fig:sim_result_p05}), \cdagger is able to detect the shift and drift more quickly than \ensemble. We find that as the \cdagger algorithm under the stationary expert becomes more noisy has increased miscoverage. This is due to the expert labels decreasing the conformal parameters, $q^{hi}, q^{lo}$ as the expert gives feedback, causing the intervals to become too small. Under the short time horizon, the miscoverage rate is higher than the desired level.
\begin{figure}[thb]
    \centering
    \includegraphics[width=0.99\linewidth]{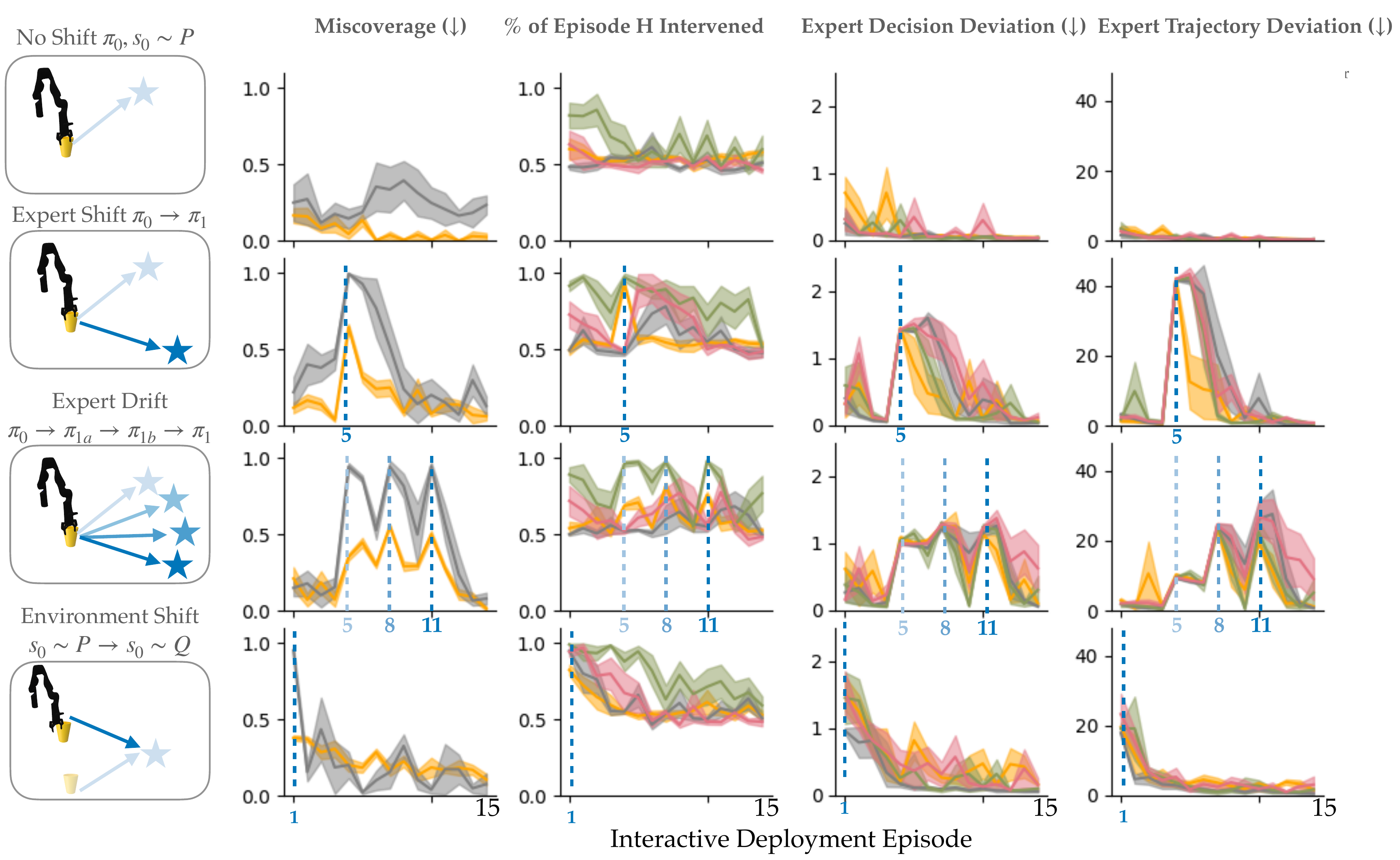}
    \caption{Under $p_t = 0.5$ intermittent feedback, \cdagger learns the shift and drift more quickly \edit{than \ensemble and \safe, with fewer queries than \lazy.}}
    \label{fig:sim_result_p05}
\end{figure}

\para{Simulated results under frequent ($p_t = 0.9$) intermittent feedback }
Under $p_t = 0.9$ intermittent feedback (Figure \ref{fig:sim_result_p09}), both algorithms are able to quickly adapt to expert shift and drift. This is because both algorithms during deployment are receiving human labels extremely frequently. Similar to partial feedback, as the \cdagger algorithm under the stationary expert becomes more noisy has increased miscoverage. \cdagger decreases the value of the conformal parameters, $q^{hi}, q^{lo}$ as the expert gives feedback, causing the intervals to become too small, giving miscoverage higher than the desired level.
\begin{figure}[thb]
    \centering
    \includegraphics[width=0.99\linewidth]{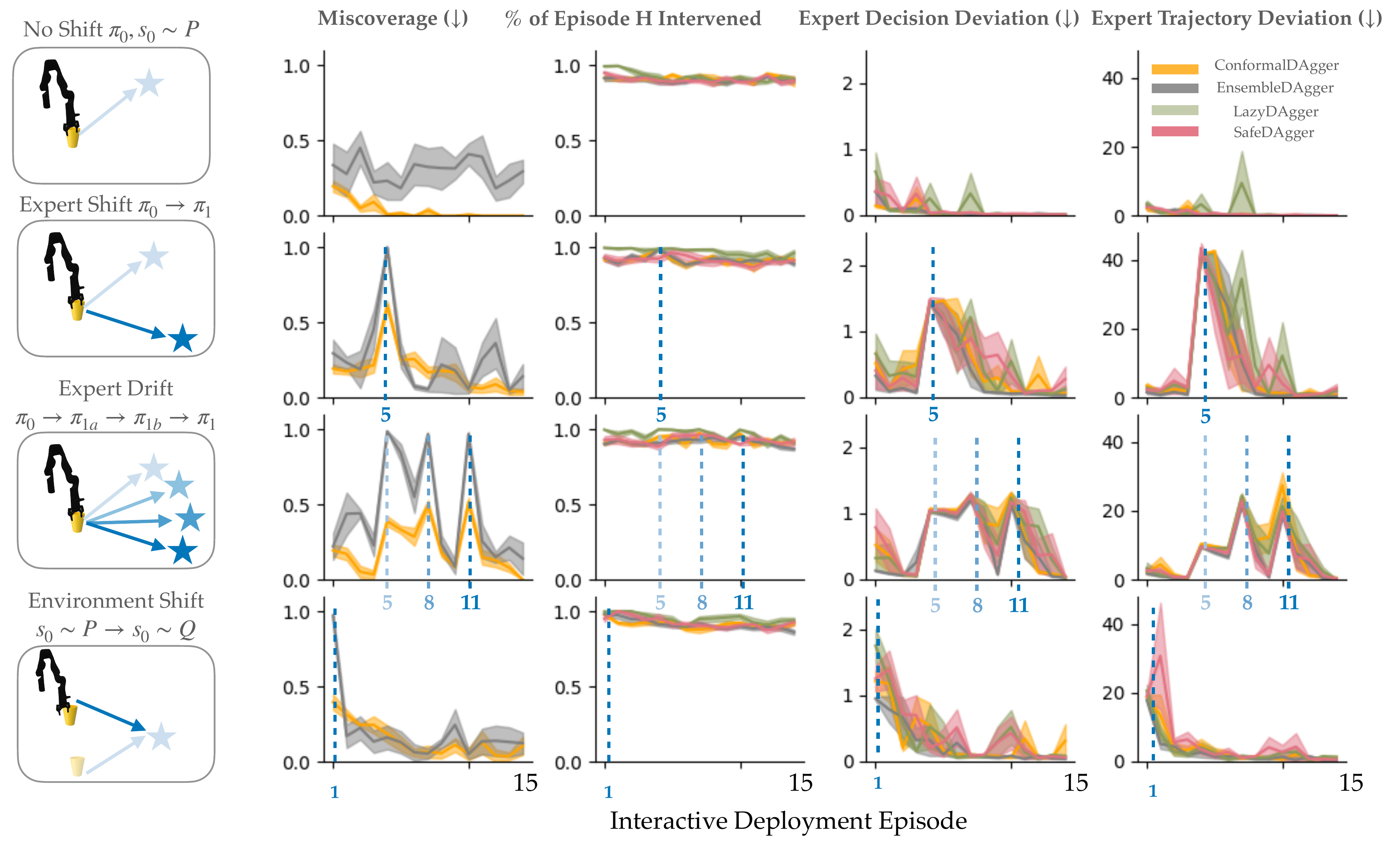}
    \caption{Under $p_t = 0.9$ intermittent feedback, all approaches learn the shift and drift quickly.}
    \label{fig:sim_result_p09}
\end{figure}

\section{Experiments in Robosuite Nut-Assembly Benchmark}
\label{sec:robomimic}

We conducted additional experiments on Robosuite’s nut-assembly task, a benchmark for LazyDAgger \cite{hoque2021lazydagger} and ThriftyDAgger \cite{hoque2021thriftydagger}. For consistency, we used the LazyDAgger thresholds from \cite{hoque2021thriftydagger} and set the SafeDAgger threshold to match the safety value used in LazyDAgger. The nut-assembly task involved a hardcoded expert policy that sequentially rotates to pick up the nut, lifts it, moves it to the peg, and lowers it to place. We trained the base policy with 30 demonstration episodes.

We evaluated this task under two conditions: a No-Shift scenario, where the task remains unchanged during deployment, and an Environment-Shift scenario, where the peg location is altered during interactive deployment (highlighted in red in Figure \ref{fig:robomimic}).

Performance metrics included: (1) autonomous success rate—success rate in deployment without a human supervisor (50 rollouts), and (2) intervention-aided success rate—success rate with a human supervisor in the loop. During training, we also tracked the number of interventions, human actions, and robot actions per episode. These metrics were calculated only for successful episodes to avoid bias from maximum episode length, which can inflate action counts for less successful policies that frequently reach the time limit.

The task required grasping a ring in a random initial pose and threading it onto a cylinder at a fixed target location. This involved two key challenges motivating learning from interventions: (1) successfully grasping the ring, and (2) accurately placing it over the cylinder (Figure \ref{fig:robomimic}). A simulated human provided interventions by teleoperating the robot. The state inputs included the robot’s joint angles and the ring’s pose, while the actions involved 3D translation, 3D rotation, and gripper control. We used 30 publicly available offline demonstrations (\url{https://github.com/ryanhoque/thriftydagger} 2,687 state-action pairs) from a human supervisor to initialize the robot policy across all algorithms.

ConformalDAgger identified the initial policy's nonperformance in both contexts, prompting more queries and achieving a higher autonomous rollout success rate compared to other baselines.

\begin{figure}[ht]
    \centering
    \includegraphics[width=0.9\linewidth]{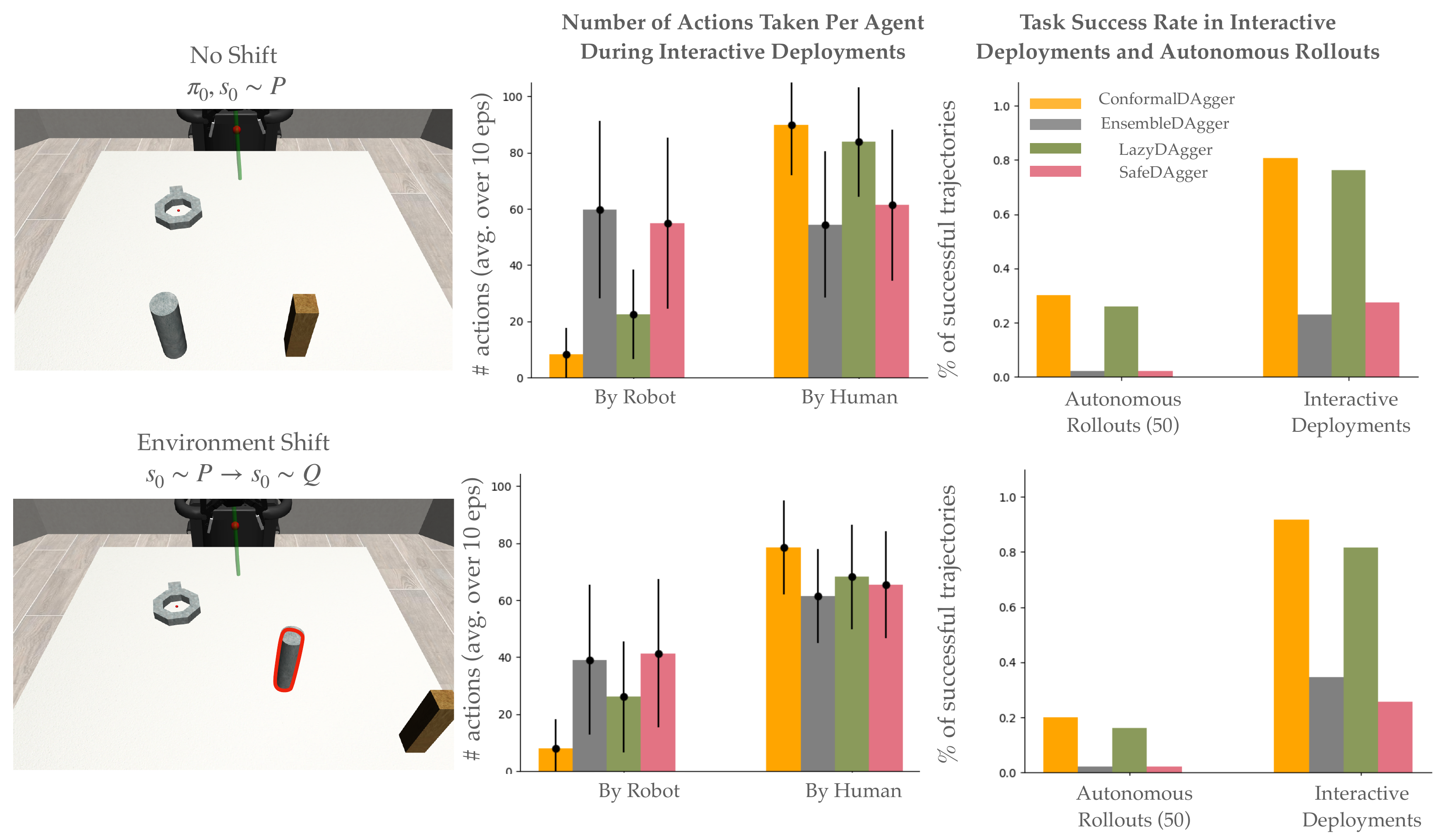}
    \caption{ConformalDAgger recognizes the learned policy from 30 demos is still nonperformant in both contexts, and increases the number of questions asked, resulting in a higher autonomous rollout success rate than other baselines.}
    \label{fig:robomimic}
\end{figure}

\section{Implementation and Task Details: Real Experiments}
\label{sec:extended_real_results}
\subsection{Learning policy training details}
\label{subsec:real_learner_training}
We record robot and human actions at 15hz. 
The initial policy $\nov_0$ is trained for 60K iterations with a batch size of 100. We use a weight decay of 1.0e-6, learning rate of 0.0001; for learning-rate scheduling, we used a cosine schedule with linear warmup \citep{nichol2021improved}. 
The number of training diffusion iterations is 100, and number of inference diffusion iterations is 16. The policy is trained on an NVIDIA RTX A6000 GPU.

\para{A look at the real-world feedback request interface}
Figure \ref{fig:real_interface} shows the real-world interface for requesting help from the user during the interactive deployment episodes. The user provides teleoperated actions via a Meta Quest 3, and the robot's uncertainty is displayed on a computer screen next to the robot. When the robot needs help, the robot pauses its execution and presents an alert notification on the screen.
\begin{figure}[thb]
    \centering
    \includegraphics[width=0.99\linewidth]{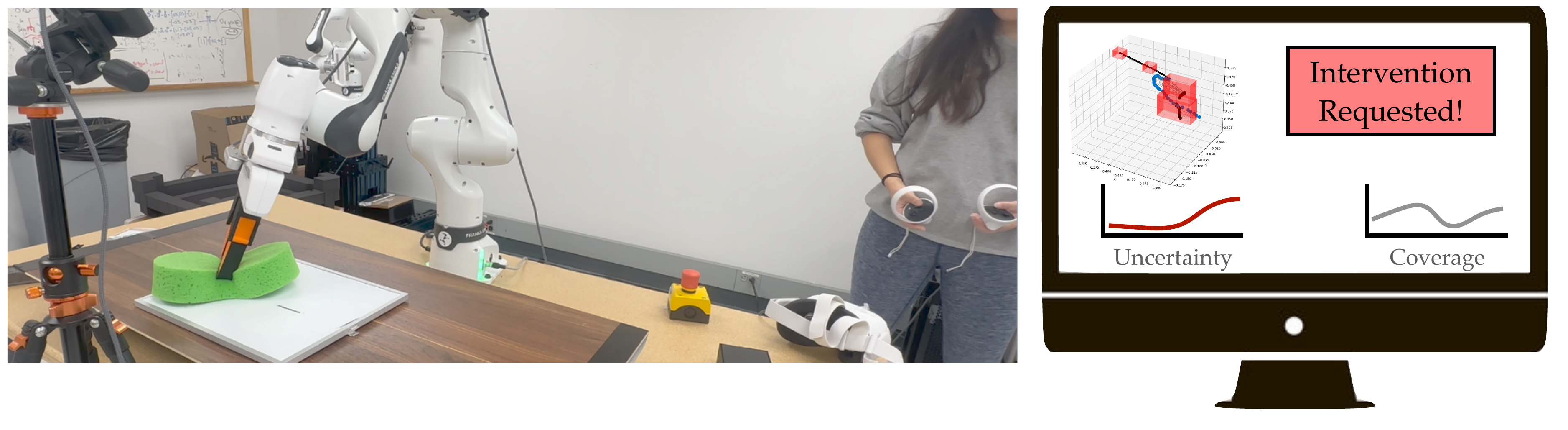}
    \caption{When the robot needs help, the robot pauses its execution and presents an alert notification on the screen.}
    \label{fig:real_interface}
\end{figure}

\para{Sponging Task Specifications}
We present the task setup details (Figure \ref{fig:sponge_task_specifications}) so that readers can also reproduce this task.
\begin{figure}[thb]
    \centering
    \includegraphics[width=0.99\linewidth]{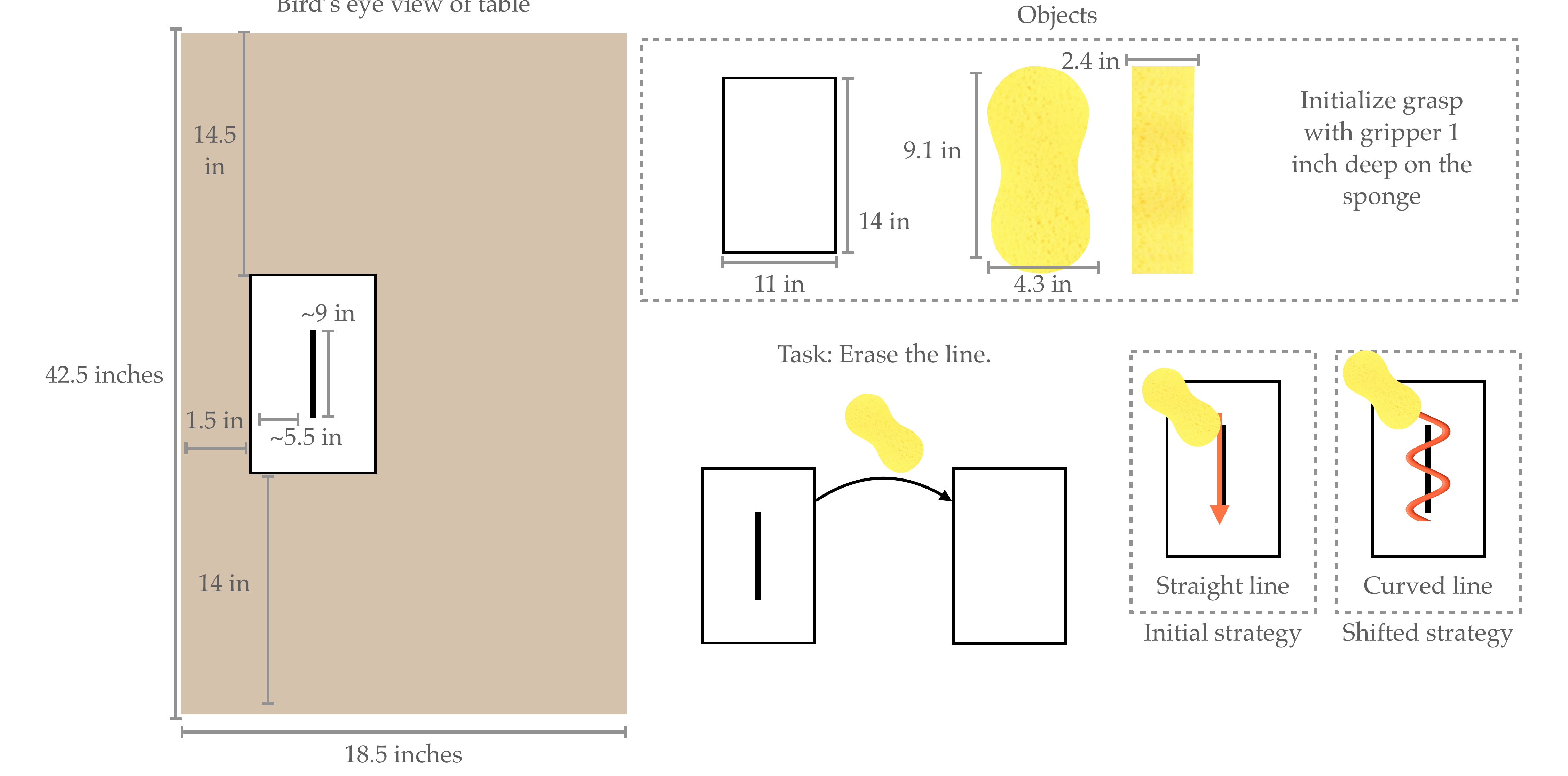}
    \caption{Sponge task specifications.}
    \label{fig:sponge_task_specifications}
\end{figure}

\end{document}